\documentclass{article}
\usepackage{amsthm}
\newtheorem{theorem}{Theorem}[section]

\newtheorem{remark}{Remark}[section]
\usepackage{arxiv}
\theoremstyle{definition}

\newtheorem{lemma}{Lemma}[section]
\newtheorem{assumption}{Assumption}[section]
\usepackage[utf8]{inputenc} 
\usepackage[T1]{fontenc}    
\usepackage{hyperref}       
\usepackage{url} 
\usepackage{booktabs}
\usepackage{multirow}
\usepackage{xcolor}
\usepackage{colortbl}
\usepackage{array}
\usepackage{amssymb}
\usepackage{comment}
\usepackage{latexsym}
\usepackage{amsmath}
\usepackage{enumitem}
\usepackage{cite}
\usepackage{multirow}
\usepackage{bm}
\usepackage{makecell}
\usepackage{soul}
\usepackage{xcolor}
\usepackage{booktabs}       
\usepackage{amsfonts}       
\usepackage{nicefrac}       
\usepackage{microtype}      
\usepackage{lipsum}
\usepackage{graphicx}
\usepackage{subcaption}
\usepackage{xcolor}

\usepackage{algorithm}
\usepackage{algorithmic}

\graphicspath{ {./images/} }

\newcommand{\be}{\begin{eqnarray}}
\newcommand{\ee}{\end{eqnarray}}

\usepackage{amsfonts} 
\usepackage{tablefootnote} 

\definecolor{codeblue}{RGB}{0, 82, 147}
\definecolor{codegreen}{RGB}{0, 128, 0}
\definecolor{codegray}{RGB}{100, 100, 100}
\definecolor{codeorange}{RGB}{230, 145, 56}
\definecolor{darkerblue}{rgb}{0,0.08,0.45}
\definecolor{royalblue}{RGB}{65,105,225}
\definecolor{lightblue}{RGB}{221,235,247}
\definecolor{fig3blue}{RGB}{47, 122, 232}
\definecolor{fig3red}{RGB}{213, 32, 52}
\definecolor{fig3green}{RGB}{0, 137, 72}
\definecolor{fig3yellow}{RGB}{217, 161, 5}
\definecolor{gray94}{gray}{.94}
\definecolor{gray90}{gray}{.90}
\definecolor{darkgreen}{RGB}{34,139,34}

\newcommand{\green}[1]{\textcolor{darkgreen}{#1}}

\newcommand{\gbf}[1]{\green{\bf{#1}}}

\newcommand{\grow}[1]{\rowcolor{gray94}{#1}}
\newcommand{\brow}[1]{\rowcolor{lightblue}{#1}}

\title{Lightweight Geometric Adaptation for Training Physics-Informed Neural Networks}

\author{
  Kang An\footnotemark[1] \\
  Department of Computational Applied Mathematics \\
  and Operations Research \\
  Rice University \\
  Houston, USA \\
  \texttt{kang.an@rice.edu}
  \And
  Chenhao Si\footnotemark[1] \\
  School of Data Science \\
  The Chinese University of Hong Kong, Shenzhen \\
  Shenzhen, China \\
  \texttt{chenhao.si@link.cuhk.edu.cn}
  \And
  Shiqian Ma \\
  Department of Computational Applied Mathematics \\
  and Operations Research \\
  Rice University \\
  Houston, USA \\
  \texttt{shiqian.ma@rice.edu}
  \And
  Ming Yan\footnotemark[2] \\
  School of Data Science \\
  The Chinese University of Hong Kong, Shenzhen \\
  Shenzhen, China \\
  \texttt{yanming@cuhk.edu.cn}
}

\begin{document}
\maketitle
\footnotetext[1]{Kang An and Chenhao Si contributed equally to this work. Their authorship order was determined randomly.}
\footnotetext[2]{Corresponding author.}

\begin{abstract}
Physics-Informed Neural Networks (PINNs) often suffer from slow convergence, training instability, and reduced accuracy on challenging partial differential equations due to the anisotropic and rapidly varying geometry of their loss landscapes. We propose a lightweight curvature-aware optimization framework that augments existing first-order optimizers with an adaptive predictive correction based on secant information. Consecutive gradient differences are used as a cheap proxy for local geometric change, together with a step-normalized secant curvature indicator to control the correction strength. The framework is plug-and-play, computationally efficient, and broadly compatible with existing optimizers, without explicitly forming second-order matrices. Experiments on diverse PDE benchmarks show consistent improvements in convergence speed, training stability, and solution accuracy over standard optimizers and strong baselines, including on the high-dimensional heat equation, Gray--Scott system, Belousov--Zhabotinsky system, and 2D Kuramoto--Sivashinsky system.

\end{abstract}


\section{Introduction}
\label{Sec intro}
Physics-Informed Neural Networks (PINNs) have emerged as a versatile framework for integrating physical laws into neural networks~\cite{raissi2019physics, karniadakis2021physics}. By incorporating governing Partial Differential Equations (PDEs) into the training objective via automatic differentiation, PINNs enable the solution of both forward and inverse problems without requiring dense observational data. This paradigm combines two complementary strengths: the expressive power of neural networks in approximating complex solution spaces, and the ability to enforce physical constraints through PDE residuals together with boundary and initial conditions. Owing to these advantages, PINNs have demonstrated effectiveness across a broad range of applications, including heat transfer~\cite{xu2023physics,cai2021physics,si2025initialization,majumdar2025hxpinn}, solid mechanics~\cite{hu2024physics,faroughi2024physics}, stochastic systems~\cite{zhang2020learning,chen2021solving}, and uncertainty quantification~\cite{yang2019adversarial,zhang2019quantifying,yang2021b}.

Despite their strong empirical promise, PINNs remain notoriously difficult to train efficiently and reliably across many practically important regimes. For stiff PDEs and problems with sharp gradients or boundary/interior layers, optimization often becomes slow and unstable, frequently leading to degraded accuracy~\cite{mcclenny2023self,wang2021understanding}. Multi-scale and high-frequency settings, such as Helmholtz equations, introduce additional challenges due to spectral bias and optimization stiffness, causing PINNs to capture low-frequency components more readily than oscillatory structures~\cite{wang2021understanding, wang2021eigenvector, krishnapriyan2021characterizing, liu2024binary}. Moreover, because the PINN objective couples PDE residuals with boundary and initial constraints, different loss components may evolve on highly disparate scales. This can lead to gradient imbalance, incomplete enforcement of physical constraints, and pronounced sensitivity to optimizer choice, sampling strategy, and hyperparameter tuning~\cite{wang2021understanding, wang2022and}. To mitigate these issues, prior work has explored adaptive loss balancing, residual-based sampling, curriculum strategies, and improved initialization~\cite{wang2021understanding, wang2022and, mcclenny2023self, anagnostopoulos2024residual, toscano2025variational, si2025convolution}. Despite meaningful progress, robust and consistently efficient PINN training remains an open challenge.

Recent studies increasingly suggest that these training difficulties are not merely empirical phenomena, but are closely tied to the geometry of the underlying optimization landscape. More broadly in deep learning, loss-landscape analyses have shown that local curvature, sharpness, and directional anisotropy play a central role in shaping optimization behavior and solution quality~\cite{li2018visualizing, zhao2025exploring}. For PINNs, this connection is even more pronounced: the coupling of PDE residuals, boundary and initial constraints, and higher-order differential operators often gives rise to highly ill-conditioned and anisotropic objectives, with optimization trajectories that traverse regions of markedly different local geometry~\cite{rathore2024challenges}. Such landscapes can contain narrow or stiff directions together with relatively flat valleys, making standard first-order updates poorly matched to the local structure and contributing to slow convergence, instability, and sensitivity to training settings. These observations indicate that improving PINN training requires not only better balancing of losses or samples, but also optimization mechanisms that can respond more effectively to the evolving local geometry of the loss landscape.

Recent studies have therefore revisited PINN training from the perspective of optimization algorithm design. Empirical and theoretical results show that optimizer choice strongly affects convergence speed, training stability, and final accuracy in PINNs~\cite{kiyani2025optimizing, urban2025unveiling}. While second-order and quasi-Newton perspectives can better capture the structure of PINN objectives~\cite{rathore2024challenges, wang2025gradient, urban2025unveiling,jnini2026curvatureawareoptimizationhighaccuracyphysicsinformed}, existing approaches often introduce substantial computational overhead, require problem-specific tuning, or remain difficult to integrate into standard training pipelines. These limitations motivate a lightweight and broadly compatible framework that can exploit local curvature information more effectively in PINN optimization.

In this work, we propose a curvature-aware optimization framework for PINNs that augments existing first-order optimizers with local geometric information. The framework is lightweight, plug-and-play, and broadly compatible, enabling improved adaptation to the anisotropic and rapidly varying landscapes that arise in PINN training. As a result, it improves convergence speed, training stability, and solution accuracy across diverse problems.
The main contributions are summarized as follows:
\begin{itemize}
    \item We highlight rapidly varying local curvature as a key optimization challenge in PINN training, and revisit PINN optimization from the perspective of matching update dynamics to the anisotropic and ill-conditioned geometry of PINN objectives.
    
    \item We propose a lightweight and plug-and-play curvature-aware optimization framework that explicitly incorporates local geometric information into the update process, enabling more effective optimization for PINNs.
    
    \item The proposed framework is broadly compatible with existing first-order optimizers and is designed to improve geometric adaptation without explicit construction of full second-order matrices, making it practical for integration into modern PINN training pipelines.
    
    \item We validate the proposed framework on diverse PDE benchmarks and demonstrate consistent improvements in convergence speed, training stability, and solution accuracy over standard optimizers and strong optimizer baselines, highlighting the benefit of explicit curvature-aware adaptation in PINNs.
\end{itemize}
The remainder of this paper is organized as follows. Section \ref{Sec PINN} briefly introduces PINNs. Section \ref{sec:instantaneous_gradient_not_enough}-\ref{sec:secant_curvature_indicator} analyze the PINN loss landscape and its associated curvature characteristics. Section \ref{sec Curvature-aware optimization framework} then introduces the proposed curvature-aware optimization framework. Section \ref{sec:convergence} provides a rigorous convergence analysis of the framework based on Adam. Finally, Section \ref{Section: Experiment} reports numerical results on several PDE benchmarks.

\section{Physics-Informed Neural Network}
\label{Sec PINN}
Denote the spatial domain as $\Omega \subset \mathbb{R}^n$ with boundary $\partial \Omega$, and let $T$ represent the time domain. The spatial-temporal variable is given by $(\mathbf{x}, t) \in \Omega \times T$. A time-dependent partial differential equation (PDE) over this domain is defined as follows:
\begin{align}
    \mathcal{F}[u](\mathbf{x}, t) &= 0, \label{(1)} \quad (\mathbf{x}, t) \in \Omega \times T, \\
    \mathcal{B}[u](\mathbf{x}, t) &= 0, \label{(2)} \quad (\mathbf{x}, t) \in \partial \Omega \times T, \quad \text{(boundary condition)} \\
    \mathcal{I}[u](\mathbf{x}, 0) &= 0, \quad \mathbf{x} \in \Omega, \hspace{51pt} \text{(initial condition)}
\end{align}
where $\mathcal{F}$, $\mathcal{B}$, and $\mathcal{I}$ are differential operators, and $u(\mathbf{x}, t)$ is the solution to the PDE, subject to boundary and initial conditions.

A PINN parameterized by $\theta$ approximates the solution $u(\mathbf{x}, t)$. The input to the neural network is $(\mathbf{x}, t)$, and the approximation is denoted by $\hat{u}(\theta)(\mathbf{x}, t)$. The PINN minimizes the following objective function:
\begin{align}
    \mathcal{L}(\theta) = \lambda_F \mathcal{L}_F(\theta) + \lambda_B \mathcal{L}_B(\theta) + \lambda_I \mathcal{L}_I(\theta), \label{(3)}
\end{align}
where
\begin{align}
    \mathcal{L}_F(\theta) &= \frac{1}{N_f} \sum_{(\mathbf{x}, t) \in \Omega_F} \big| \mathcal{F}[\hat{u}(\theta)](\mathbf{x}, t) \big|^2, \label{(4)} \\
    \mathcal{L}_B(\theta) &= \frac{1}{N_b} \sum_{(\mathbf{x}, t) \in \Omega_B} \big| \mathcal{B}[\hat{u}(\theta)](\mathbf{x}, t) \big|^2, \label{(5)} \\
    \mathcal{L}_I(\theta) &= \frac{1}{N_0} \sum_{(\mathbf{x}, 0) \in \Omega_I} \big| \mathcal{I}[\hat{u}(\theta)](\mathbf{x}, 0) \big|^2. \label{(6)}
\end{align}
Here, $\Omega_F$, $\Omega_B$, and $\Omega_I$ are the training sets for the PDE residual, boundary condition, and initial condition, respectively, with cardinality $N_f$, $N_b$, and $N_0$. The weights $\lambda_F$, $\lambda_B$, and $\lambda_I$ are hyperparameters tuning the contributions of each loss component. Notably, $\Omega_F$ may include points on the boundary or at the initial time, allowing $\Omega_F \cap \Omega_B$ and $\Omega_F \cap \Omega_I$ to be non-empty.

\section{Why the instantaneous gradient is not enough in PINNs}
\label{sec:instantaneous_gradient_not_enough}

The optimization difficulty of PINNs is rooted not only in gradient magnitude imbalance, but also in the highly heterogeneous geometry induced by stiff PDE residuals \cite{kiyani2025optimizing,rathore2024challenges}. When low-frequency boundary or initial-condition terms interact with high-frequency residual constraints, the resulting loss landscape often becomes strongly anisotropic: some directions remain relatively flat, while others are sharply curved and rapidly varying. In such regions, even a small parameter displacement may lead to a noticeable change in the gradient field, making the optimization dynamics highly sensitive to local geometry \cite{kiyani2025optimizing,krishnapriyan2021characterizing,anagnostopoulos2024residual,rathore2024challenges}.

\begin{figure}[!htbp]
    \centering
    \begin{minipage}{0.49\textwidth}
        \centering
        \includegraphics[width=\linewidth, trim=0cm 1cm 0cm 1cm, clip]{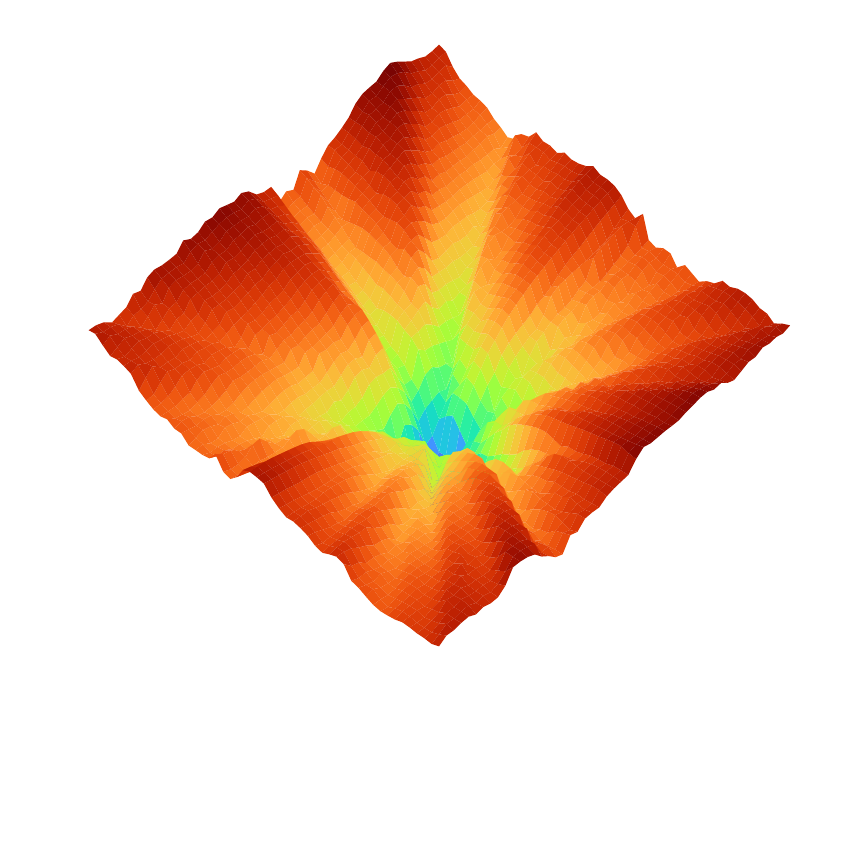}
    \end{minipage}
    \hfill 
    \begin{minipage}{0.49\textwidth}
        \centering
        \includegraphics[width=\linewidth, trim=0cm 1cm 0cm 1cm, clip]{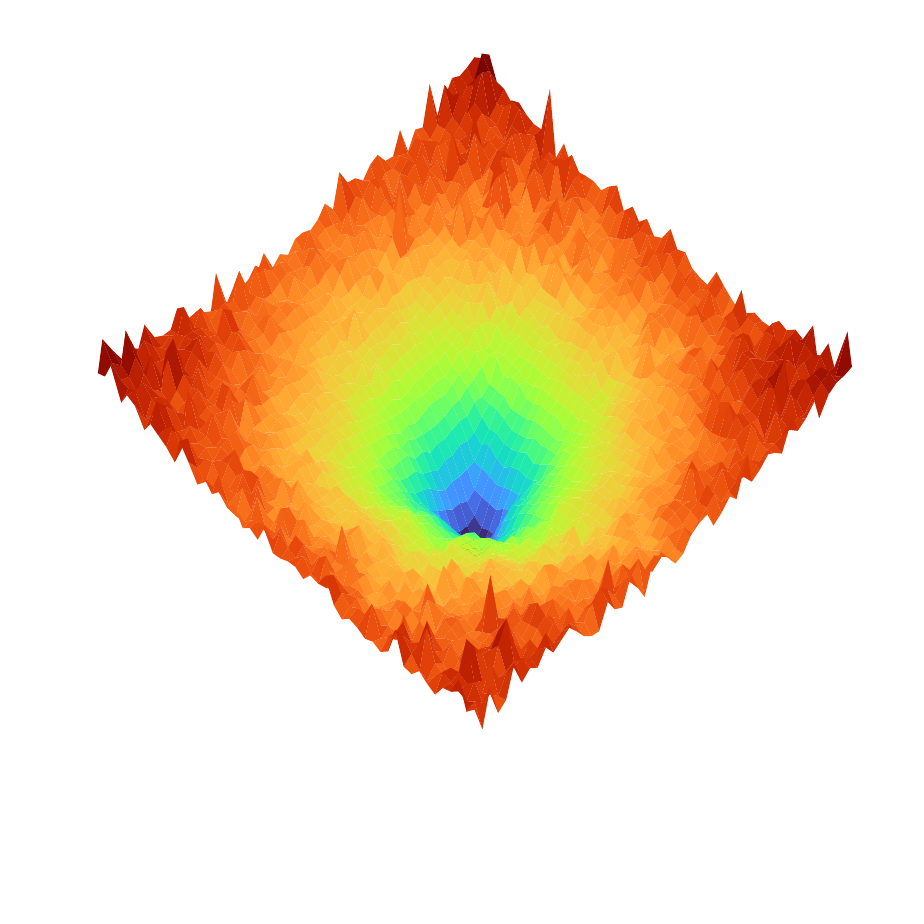}
    \end{minipage}
    \caption{Projection of the loss landscape $\mathcal{L} = \mathcal{L}_{I} + \lambda \mathcal{L}_{F}$ for the viscous Burgers' equation onto a 2D parameter subspace using random directions. The vertical axis shows the logarithm of the scalar loss value, while the two horizontal axes represent the perturbation scale along the two random directions. The left subfigure shows the landscape for the data loss alone ($\lambda = 0$), consisting of the initial and boundary condition constraints. The right subfigure displays the PDE residual landscape, illustrating the highly non-convex and rugged topography introduced by the physics constraints.}
    \label{fig:landscape}
\end{figure}

This effect is illustrated in Fig.~\ref{fig:landscape}. Compared with the relatively smooth basin of a purely data-driven objective, the inclusion of the stiff PDE residual produces a much more rugged landscape with sharper valleys and more rapidly varying geometry. The main difficulty is not nonconvexity alone, but the coexistence of stiff and flat directions: a step size small enough to remain stable in sharp directions can become too conservative to make meaningful progress in flatter ones.

In this regime, relying only on the instantaneous gradient is often insufficient \cite{zhao2025exploring}. Furthermore, standard momentum—which relies on an exponential moving average (EMA) of past gradients—is fundamentally reactive. It accumulates a lagging historical average that cannot anticipate sudden geometric changes. For PINNs, where local stiffness can vary substantially along the optimization path, it is therefore useful to supplement the current gradient with predictive information that reflects recent geometric change. A natural question is what additional signal can provide useful look-ahead information about such local geometric variation without incurring the cost of second-order methods.

A simple and practical choice is the consecutive gradient difference
\begin{align}
y_k := g_k - g_{k-1},
\end{align}
where $g_k = \nabla \mathcal{L}(\theta_k)$ is the gradient at the current iterate. Unlike the instantaneous gradient, which only reflects the local descent direction at $\theta_k$, the difference between consecutive gradients captures how the gradient field has changed along the recent optimization trajectory. In this sense, it provides a lightweight trajectory-level signal of recent local geometric variation.

This signal is attractive because it naturally supports a predictive correction: if the optimization trajectory continues its recent trend, then the consecutive gradient difference provides a first-order surrogate for how the gradient may evolve over a short horizon. Related ideas also appear in optimistic, extragradient-type, and gradient-difference-based accelerated methods \cite{daskalakis2025limitpointsoptimisticgradient,rakhlin2013optimizationlearninggamespredictable,mertikopoulos2018optimisticmirrordescentsaddlepoint,mokhtari2019unifiedanalysisextragradientoptimistic,JMLR:v17:15-084,xie2024adan,nesterov1983method,nesterov2013introductory}. For our purposes, the key point is that $y_k$ offers a cheap trajectory-level proxy for short-horizon future gradient variation. In PINNs, however, this signal should not be used with a fixed strength: due to strong anisotropy and local stiffness, the same extrapolative correction may help in flatter regions but become overly aggressive in sharper ones. The practical issue is therefore not only whether gradient difference is informative, but also when it should be trusted and how strongly it should influence the update. This suggests that the predictive correction should be modulated according to the recent directional geometry of the optimization trajectory.

\subsection{Secant-based gating for predictive correction}
\label{sec:secant_curvature_indicator}
The previous subsection motivates the need for predictive information beyond the instantaneous gradient in PINNs. The remaining question is how strongly such predictive information should influence the update throughout training. In particular, if consecutive gradient difference is used as a low-cost extrapolative signal, its utility need not be uniform across iterations. Before introducing a geometry-dependent modulation mechanism, we first illustrate this variability through a realized one-step diagnostic.

To quantify the realized utility of secant-style predictive correction, we define the relative secant prediction error
\begin{align}
R_k
:=
\frac{
\|\nabla \mathcal L(\theta_{k+1})-(g_k+\tau_k y_k)\|_2
}{
\|\nabla \mathcal L(\theta_{k+1})-g_k\|_2
},
\label{eq:error_ratio}
\end{align}
where $g_k:=\nabla \mathcal L(\theta_k)$ is the current gradient, $y_k:=g_k-g_{k-1}$ is the consecutive gradient difference, and $\tau_k$ denotes the effective extrapolative strength along the recent trajectory. This ratio compares the secant-corrected surrogate $g_k+\tau_k y_k$ against the baseline surrogate $g_k$ in predicting the next-step gradient. Smaller values of $R_k$ indicate better realized predictive utility of the secant correction, while larger values indicate that the correction is becoming less reliable. In particular, $R_k>1$ means that the secant-corrected surrogate is less accurate than  $g_k$ alone.

Figure~\ref{fig:kappa} shows the histories of this diagnostic for AdamW. The main observation is that the realized value of predictive correction varies substantially during PINN training: the error exhibits pronounced fluctuations over time, and some iterations enter the regime $R_k>1$. This indicates that a fixed extrapolative strength is not uniformly appropriate. In other words, the practical issue is not whether gradient difference can provide useful look-ahead information, but under what local geometry it should contribute more strongly or more conservatively to the update.

\begin{figure}[!htb]
\centering
    \begin{minipage}{0.45\textwidth}
     \centering
     \includegraphics[width=\linewidth]{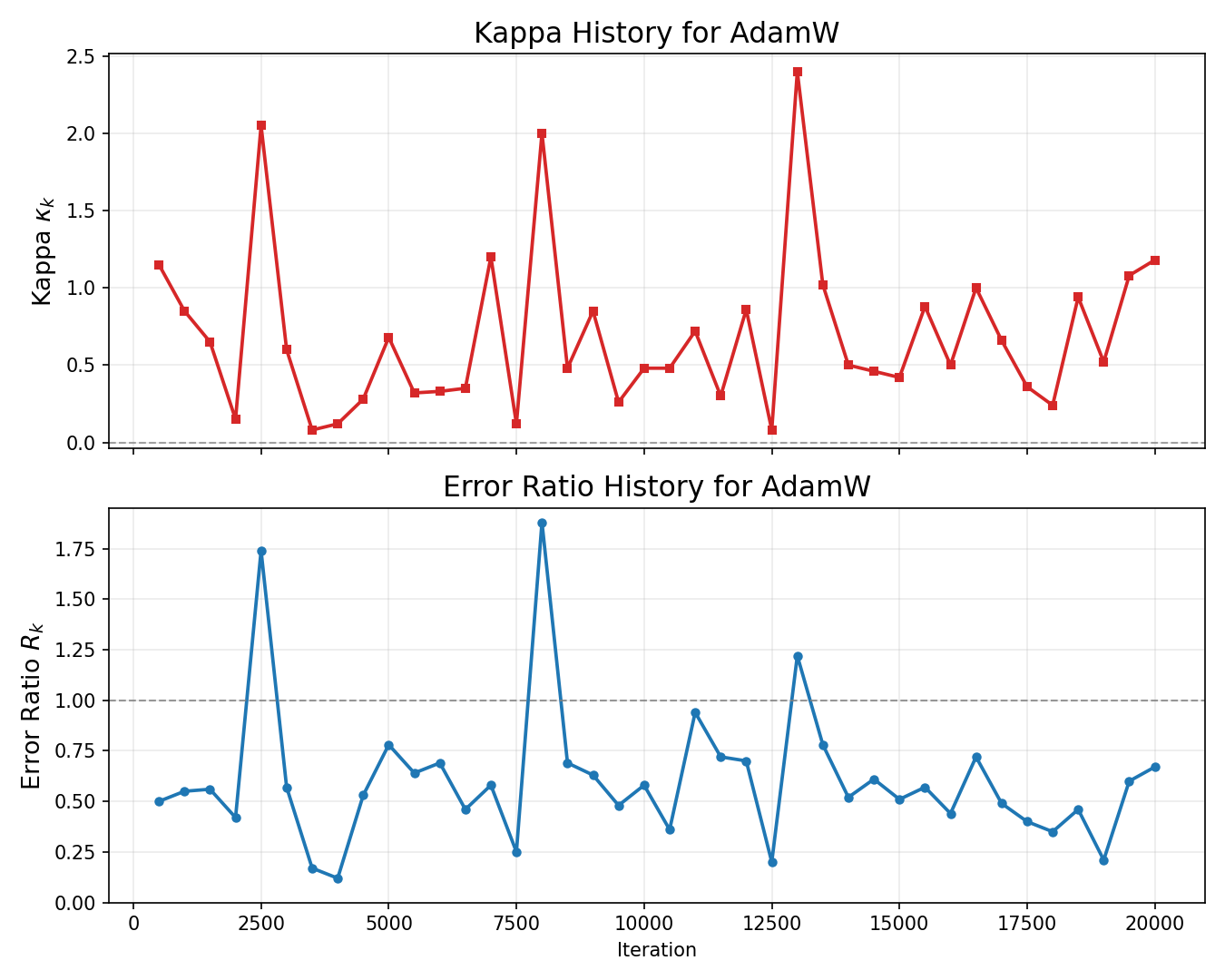} 
   \end{minipage} 
   \begin{minipage}{0.45\textwidth}
     \centering
     \includegraphics[width=\linewidth]{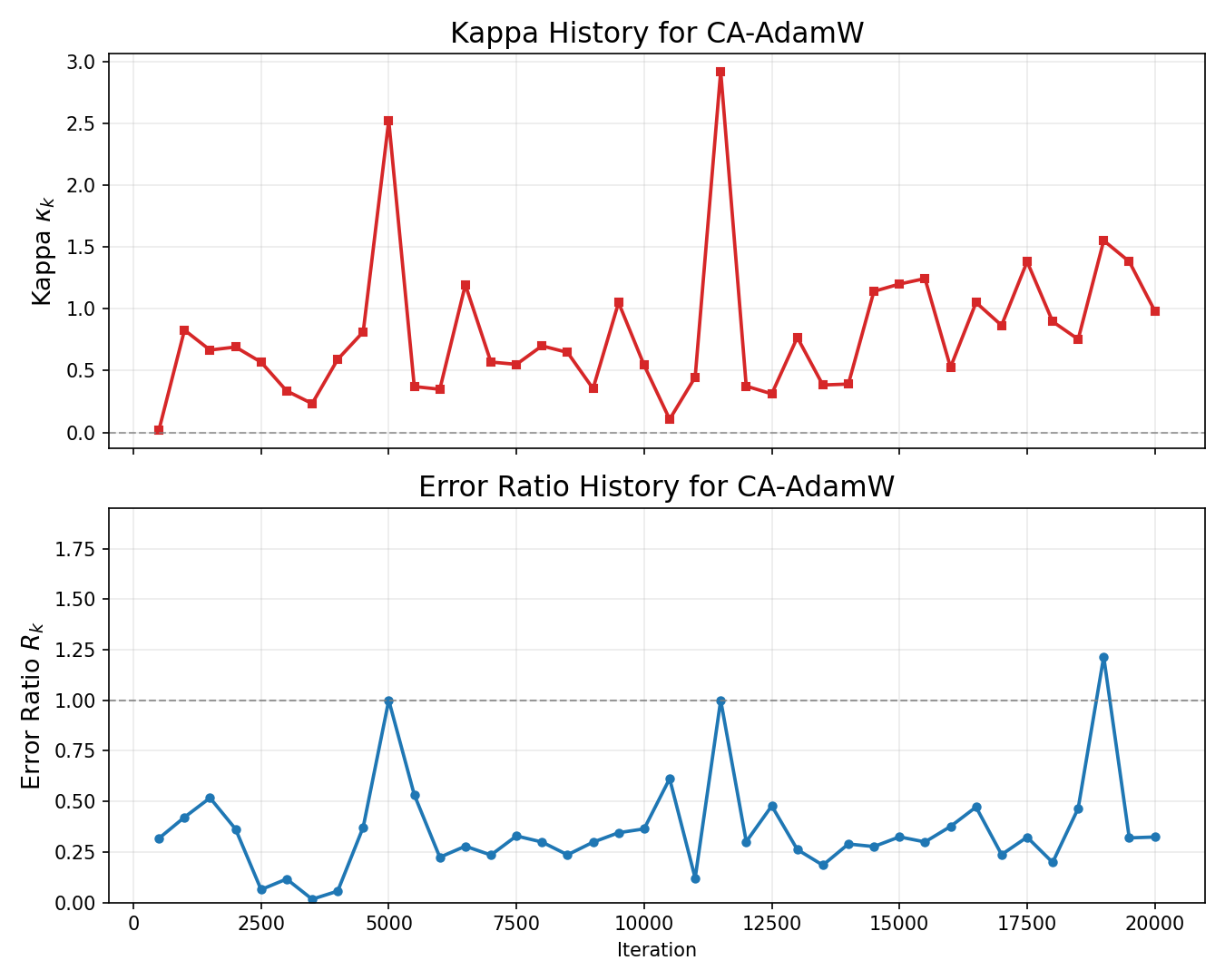} 
   \end{minipage} 

   \caption{Histories of the secant curvature indicator $\kappa_k$ (top) and the realized secant prediction error ratio $R_k$ (bottom) for AdamW (left) and CA-AdamW (right). While CA-AdamW does not consistently reduce the directional curvature indicator itself, it substantially suppresses large error-ratio spikes and keeps $R_k$ below or near the baseline-failure regime for a larger portion of training. This supports the intended role of curvature-aware gating, which smoothly modulates the strength of secant-based predictive correction according to the local directional geometry.}
\label{fig:kappa}
\end{figure}

To understand the geometric source of this variability, we examine how the gradient difference $y_k$ acts as a predictive signal for the next-step gradient. To formalize the recent trajectory, we define the recent step displacement alongside our gradient difference:
\begin{align}
s_k := \theta_k-\theta_{k-1},
\qquad
y_k := g_k-g_{k-1}.
\label{eq:secant_pair}
\end{align}

By the fundamental theorem of calculus applied to the gradient field along the segment from $\theta_{k-1}$ to $\theta_k$, we have
\begin{align*}
y_k
=
\int_0^1 H(\theta_{k-1}+t s_k)\,s_k\,dt,
\end{align*}
where $H(\theta)=\nabla^2\mathcal{L}(\theta)$ is the Hessian. Defining the segment-averaged Hessian $\bar H_k := \int_0^1 H(\theta_{k-1}+t s_k)\,dt$, we obtain:
\begin{align*}y_k=\bar H_k s_k.
\end{align*}
Hence, the consecutive gradient difference is the secant response of the gradient field along the recent trajectory direction, and can be interpreted as a low-cost directional curvature signal.

To apply this signal to the upcoming step, define the next step displacement $d_{k+1}:= \theta_{k+1} - \theta_k$. We relate this new displacement to our recent trajectory by writing: \begin{align*}d_{k+1}=\tau_k s_k+\delta_k,\end{align*}
where $\tau_k$, as introduced above, measures how much the next step continues the recent trajectory direction, and $\delta_k =d_{k+1}-\tau_k s_k$ measures the deviation from that local directional persistence.

Now, a first-order expansion at $\theta_k$ gives the next-step gradient:
\begin{align*}\nabla \mathcal{L}(\theta_{k+1}) = g_k + H(\theta_k)d_{k+1}+ o(\|d_{k+1}\|_2).\end{align*}
Substituting our trajectory decomposition into this expansion and using $y_k = \bar H_k s_k$ yields:
\begin{align}\nabla \mathcal{L}(\theta_{k+1}) = g_k+\tau_k y_k +\tau_k\bigl(H(\theta_k)-\bar H_k\bigr)s_k +H(\theta_k)\delta_k +o(\|d_{k+1}\|_2).
\label{eq:predictive_error_decomposition}
\end{align}
This equation provides an explanation for the variability observed in the realized error $R_k$. The term $g_k+\tau_k y_k$ is the desired predictive component; it uses the recent secant response to approximate the leading variation of the next-step gradient. The remaining terms explain why the realized predictive quality can change substantially across iterations, even when the gradient difference remains available at every step.

More specifically, these remaining terms isolate two distinct sources of approximation error. The term $\tau_k\bigl(H(\theta_k)-\bar H_k\bigr)s_k$ captures curvature mismatch: it remains small when the curvature along the recent segment is consistent with the current iterate but grows large when the local geometry changes rapidly between consecutive steps. The term $H(\theta_k)\delta_k$ captures the penalty for directional deviation: it is small when the upcoming step remains coherent with the recent trajectory and large when the optimizer shifts course.

Crucially, this directional penalty is not determined by the deviation $\delta_k$ alone; it is scaled by the Hessian at the current point. As a result, identical shifts in trajectory can yield vastly different prediction errors depending on the local geometry's stiffness. In regions of strong local curvature, even a moderate deviation from the recent path is amplified into a severe secant-extrapolation error.

This observation is the cornerstone of our proposed method. It suggests that the role of local geometry is not to dictate whether the secant signal is useful, but rather to determine how strongly it should influence the current update. In regions of mild curvature variation where the recent direction remains predictive, the secant correction can be applied with confidence. Conversely, when the local geometry is stiff or rapidly changing, the identical correction term becomes aggressive and prone to overshooting, necessitating a reduced coefficient. This dynamic is especially critical in PINNs, whose loss landscapes are notorious for exhibiting strong anisotropy and stiff directional geometry due to the interplay of heterogeneous training objectives and PDE residual constraints \cite{kiyani2025optimizing,rathore2024challenges,krishnapriyan2021characterizing}. Ultimately, \eqref{eq:predictive_error_decomposition} frames the practical challenge as an extrapolation-gain control problem.

However, the theoretical quantities in \eqref{eq:predictive_error_decomposition} cannot be computed online. The diagnostic ratio $R_k$ relies on the unobserved future gradient $\nabla \mathcal{L}(\theta_{k+1})$, and the decomposition itself demands explicit Hessian evaluations that are computationally prohibitive during training. Therefore, the algorithm requires a lightweight, online surrogate—a metric that continuously measures local geometry to dynamically scale the secant correction along the recent trajectory.

For this purpose, we introduce the secant curvature indicator:
\begin{align*}\kappa_k := \frac{s_k^\top y_k}{\|s_k\|_2^2}=\frac{s_k^\top \bar H_k s_k}{\|s_k\|_2^2},\end{align*}
which is the Rayleigh quotient of the segment-averaged Hessian along the recent displacement direction. Thus, $\kappa_k$ measures the directional curvature actually experienced by the optimizer along its latest trajectory segment.

The role of $\kappa_k$ is not to precisely estimate the future error in \eqref{eq:error_ratio}, but rather to serve as a cheap, online indicator of how aggressively secant-based predictive correction should be applied. When $\kappa_k$ is small, the recent direction exhibits mild curvature, allowing the secant response $y_k$ to be incorporated more strongly without making the update excessively reactive. Conversely, when $\kappa_k$ is large, the recent direction is stiff. In this regime, applying a fixed extrapolative gain would amplify curvature and directional mismatches, likely leading to overshoot. Therefore, a larger $\kappa_k$ dictates that the correction coefficient should be damped.

We therefore construct the corrected gradient as 
\begin{align*}
\tilde g_k = g_k + \alpha_k y_k,
\end{align*}
where the gating coefficient $\alpha_k$ is a monotonically decreasing function of $\kappa_k$. In our implementation, we use
\begin{align*}
\alpha_k = \alpha_{\mathrm{base}}\bigl(1+\tanh(-\kappa_k)\bigr),
\end{align*}
where $\alpha_{\mathrm{base}}>0$ is the nominal correction level. 
By smoothly modulating $\alpha_k$, the secant correction is strengthened when the local geometry supports a bold predictive step, and progressively suppressed when local stiffness warns against aggressive extrapolation.

The optimizer-level effect of this geometry-aware modulation is illustrated in Fig.~\ref{fig:kappa}. As shown in the right panel, compared with AdamW, our curvature-aware variant produces a substantially more stable error-ratio trajectory with fewer and smaller spikes, even when elevated values of $\kappa_k$ arise during training. This behavior confirms the intended role of curvature-aware gating: smoothly adapting the strength of the predictive correction to the local directional geometry, rather than blindly applying a fixed extrapolative gain throughout optimization.

\subsection{Curvature-aware optimization framework}
\label{sec Curvature-aware optimization framework}
While related in spirit to Adan's use of gradient difference, our framework targets the specific geometric heterogeneity of PINNs. Rather than relying on a fixed extrapolation form or violently tuning the optimizer's internal EMA hyperparameters step-by-step, we modulate the look-ahead signal through a trajectory-level secant gate.

We now instantiate the proposed mechanism within AdamW to illustrate its practical implementation. Although the proposed curvature-aware trajectory correction is optimizer-agnostic, we use AdamW as a representative example because its update structure makes the integration especially transparent. The resulting optimizer, referred to as \textbf{CA-AdamW}, preserves the original AdamW update rule while replacing the raw gradient input with a curvature-gated boosted gradient. In this way, the method remains lightweight and fully first-order, while incorporating the trajectory-level correction developed in the previous subsections.

\begin{algorithm}[!ht]
   \caption{Curvature-Aware AdamW (CA-AdamW)}
   \label{alg:CAAdamW}
\begin{algorithmic}[1]
    \STATE {\bfseries Input:} initialization $\theta^0$, number of iterations $T$, $\beta_1$, $\beta_2$, $\beta_a$, $\delta$, $\alpha_{\mathrm{base}}$, $\lambda$, $\{\eta^k\}_{k=1}^T$
    \STATE Initialize $m^0 \leftarrow \mathbf{0}$, $v^0 \leftarrow \mathbf{0}$, $a^0 \leftarrow \mathbf{0}$, $g^{-1} \leftarrow \mathbf{0}$
    \FOR{$k=1$ {\bfseries to} $T$}
        \STATE $g^k = \frac{1}{|\mathcal{Z}^k|}\sum_{\zeta_j \in \mathcal{Z}^k} \nabla \ell(\theta^{k-1}; \zeta_j)$
        \IF{$k = 1$}
            \STATE $\alpha^k = 0$
            \STATE $\tilde{g}^k = g^k$
        \ELSE
            \STATE $a^k = \beta_a a^{k-1} + (1-\beta_a)(g^k-g^{k-1})$
            \STATE $s^k = \theta^{k-1}-\theta^{k-2}$
            \STATE $\kappa^k = \frac{\langle s^k, g^k-g^{k-1}\rangle}{\|s^k\|_2^2}$
            \STATE $\alpha^k = \alpha_{\mathrm{base}}(1+\tanh(-\kappa^k))$
            \STATE $\tilde{g}^k = g^k+\alpha^k a^k$
        \ENDIF
        \STATE $m^k = \beta_1 m^{k-1} + (1-\beta_1)\tilde{g}^k$
        \STATE $v^k = \beta_2 v^{k-1} + (1-\beta_2)\tilde{g}^k \odot \tilde{g}^k$
        \STATE $\hat{m}^k = \frac{m^k}{1-\beta_1^k}$, $\hat{v}^k = \frac{v^k}{1-\beta_2^k}$
        \STATE $\theta^k = \theta^{k-1} - \eta^k \left( \frac{\hat{m}^k}{\sqrt{\hat{v}^k}+\delta} + \lambda \theta^{k-1} \right)$
    \ENDFOR
\end{algorithmic}
\end{algorithm}

Algorithm~\ref{alg:Adam_NC} proceeds in three stages. First, it constructs recent trajectory information from consecutive iterations. The parameter displacement
$s_k = \theta_k - \theta_{k-1}$
captures the latest update direction, while the correction signal is formed using an exponential moving average
\begin{equation}
a_k = \beta_a a_{k-1} + (1-\beta_a)(g_k-g_{k-1}),
\end{equation}
where $\beta_a \in [0,1)$ is a decay factor. Compared to the raw consecutive gradient difference, this EMA-smoothed form provides a more stable estimate of recent gradient variation under stochastic training. This stability is particularly beneficial in PINNs, given collocation and residual-sampling noise.

Second, the algorithm computes the secant curvature indicator $\kappa_k$ and maps it to the adaptive boost coefficient
\begin{equation}
\alpha_k
=
\alpha_{\mathrm{base}}
\left(1+\tanh(-\kappa_k)\right).
\end{equation}
We let $\alpha_0=0$ since $s_0=0$. Here, the raw secant difference is retained in $\kappa_k$ to preserve its direct geometric interpretation, whereas the EMA-smoothed signal $a_k$ is used in the correction term for improved robustness. This realizes the curvature-aware gating mechanism introduced previously: a stronger correction is applied when the recent trajectory exhibits a flat directional response, while the correction is sharply suppressed when local stiffness indicates a higher risk of amplifying oscillations.

Finally, the boosted gradient
\begin{equation}
\tilde g_k = g_k + \alpha_k a_k
\end{equation}
is used in place of the raw gradient in the AdamW moment updates. That is, both the first and second moments are updated using $\tilde g_k$ rather than $g_k$. As a result, the base optimizer receives a gradient signal that combines the current descent information with a smoothed, curvature-gated estimate of recent gradient evolution. The underlying AdamW machinery remains unchanged; only its gradient input is enriched by the proposed trajectory-level correction.

Although we present CA-AdamW here for concreteness, the proposed CA framework is not specific to AdamW. Its core operations, constructing a smoothed gradient-difference signal, computing a secant-curvature gate, and injecting the boosted gradient into the optimizer input, carry over directly to other first-order methods. Corresponding instantiations for SOAP and Muon are provided in Algorithm~\ref{alg:CA-SOAP} and \ref{alg:CA-Muon}, and all three optimizer families are evaluated in Section \ref{Section: Experiment}.

\section{Convergence Analysis of Curvature-Aware Adam}
\label{sec:convergence}

To facilitate theoretical analysis, we consider a simplified variant of CA-AdamW that omits weight decay, bias correction, and the EMA smoothing of the gradient-difference term, while retaining the core curvature-aware gating mechanism introduced in Section~\ref{sec Curvature-aware optimization framework}. Specifically, the smoothed correction signal $a_k$ is replaced by the raw consecutive gradient difference $g_k-g_{k-1}$.

We consider the stochastic optimization problem
\begin{equation}
\min_{\theta\in\mathbb{R}^d} F(\theta):=\mathbb{E}_{\zeta}[\ell(\theta;\zeta)].
\end{equation}
Let $g_k$ be a stochastic gradient estimator of $\nabla F(\theta_k)$. The simplified recursion is
\begin{align}
\tilde g_k
&:=
g_k+\alpha_k(g_k-g_{k-1}),
\qquad
&m_k
&:=
\beta_1 m_{k-1}+(1-\beta_1)\tilde g_k,\\
v_k
&:=
\beta_2 v_{k-1}+(1-\beta_2)\tilde g_k^{2},
\qquad
&\theta_{k+1}
&:=
\theta_k-\eta \frac{m_k}{\sqrt{v_k}+\delta}.
\label{eq:simplified-update}
\end{align}
Throughout the analysis, we only use the uniform bound $|\alpha_k|\le C_\alpha:=2\alpha_{\mathrm{base}}$. For notational convenience, we define the diagonal preconditioning matrix:
\begin{equation}
D_k:=\operatorname{diag}\!\left(\frac{1}{\sqrt{v_k}+\delta}\right),
\qquad
\theta_{k+1}-\theta_k=-\eta D_k m_k.
\end{equation}

For the analysis, we adopt several standard, relatively mild assumptions commonly used in the literature on optimization methods in Deep Learning \cite{defossez2022a,xie2024adan,an2025asgo,guo2021novel,kwon2021asam}. These assumptions provide a conventional foundation for establishing the convergence results that follow.
\begin{assumption}[$L$-smoothness]
\label{assump:smooth}
The objective function $F$ is continuously differentiable and has an $L$-Lipschitz continuous gradient, satisfying:
\begin{equation}
\|\nabla F(x)-\nabla F(y)\| \le L\|x-y\|,
\qquad \forall x,y \in \mathbb{R}^d.
\end{equation}
\end{assumption}

\begin{assumption}[Stochastic Oracle Properties]
\label{assump:oracle}
The stochastic gradient estimators $\{g_k\}$ satisfy the following conditions for all $k$:
\begin{enumerate}
    \item Unbiasedness and bounded variance:
    \begin{align}
        \mathbb{E}[g_k \mid \theta_k] = \nabla F(\theta_k), \qquad \mathbb{E}\big[\|g_k-\nabla F(\theta_k)\|^2 \mid \theta_k\big] \le \sigma^2.
    \end{align}
    \item Uniformly bounded gradients: There exists a constant $G>0$ such that
    \begin{equation}
        \|g_k\|_\infty \le \|g_k\| \le G.
    \end{equation}
\end{enumerate}
\end{assumption}

The following elementary estimates will be used for the convergence analysis.
\begin{lemma}[Properties of modified gradients] 
\label{lem:modified-grad-props}
Under Assumptions~\ref{assump:smooth} and~\ref{assump:oracle}, for all $k\geq 0$, the modified stochastic gradient $\tilde g_k$ and the modified full gradient 
$\tilde \nabla_k := \nabla F(\theta_k) + \alpha_k\big(\nabla F(\theta_k)-\nabla F(\theta_{k-1})\big)$ satisfy:
\begin{align}
\|\tilde g_k\|_\infty \le \|\tilde g_k\| &\le (1+2C_\alpha)G, \\
\|\tilde \nabla_k - \nabla F(\theta_k)\| &\le C_\alpha L \|\theta_k-\theta_{k-1}\|.
\end{align}
\end{lemma}

\begin{proof}
By definition, we have $\tilde g_k = g_k + \alpha_k(g_k-g_{k-1})$. Taking the norm and applying the triangle inequality yields:
\begin{equation*}
\|\tilde g_k\|
\le
\|g_k\| + |\alpha_k|\big(\|g_k\| + \|g_{k-1}\|\big)
\le
(1+2C_\alpha)G,
\end{equation*}
where we have used the uniform bound $|\alpha_k| \le C_\alpha$ and the assumption that $\|g_k\| \le G$.

For the second claim, by the definition of $\tilde\nabla_k$ and Assumption~\ref{assump:smooth}, we have:
\begin{equation*}
\|\tilde \nabla_k - \nabla F(\theta_k)\|
=
|\alpha_k|\,\|\nabla F(\theta_k)-\nabla F(\theta_{k-1})\|
\le
C_\alpha L\|\theta_k-\theta_{k-1}\|.
\end{equation*}
This completes the proof.
\end{proof}

\begin{lemma}[Bounded displacement and preconditioner]
\label{lem:bounded-displacement}
Considering Assumptions~\ref{assump:smooth} and~\ref{assump:oracle}, for all $k\geq 0$, the diagonal preconditioning matrix $D_k$ satisfies:
\begin{equation}\label{eq:dia}
c_0 I \preceq D_k \preceq  \frac{1}{\delta} I, \qquad \text{where} \quad c_0 := \frac{1}{(1+2C_\alpha)G+\delta},
\end{equation}
and the iterate displacement is bounded by:
\begin{equation}
\|\theta_{k+1}-\theta_k\| \le \eta M,
\qquad \text{where} \quad
M := \frac{(1+2C_\alpha)G}{\delta}.
\end{equation}
\end{lemma}

\begin{proof}
First, unrolling the momentum recursion $m_k = \beta_1 m_{k-1} + (1-\beta_1)\tilde g_k$ with the initialization $m_{-1}=0$, and applying Lemma~\ref{lem:modified-grad-props}, an induction yields $\|m_k\| \le (1+2C_\alpha)G$. 

Similarly, because the second moment accumulator $v_k$ tracks the exponential moving average of $\tilde g_k^2$, the infinity norm bound from Lemma~\ref{lem:modified-grad-props} ensures that $\|v_k\|_\infty \le (1+2C_\alpha)^2 G^2$. Taking the square root coordinate-wise gives $0 \le \sqrt{v_k} \le (1+2C_\alpha)G$. By the definition of the preconditioner $D_k = \operatorname{diag}(1/(\sqrt{v_k} + \delta))$, these coordinate-wise bounds directly yield the matrix inequalities in \eqref{eq:dia}.

Finally, taking the norm of the parameter update rule $\theta_{k+1}-\theta_k = -\eta D_k m_k$ and combining these intermediate bounds yields:
\begin{equation*}
\|\theta_{k+1}-\theta_k\|
=
\eta\|D_k m_k\|
\le
\eta\|D_k\|\,\|m_k\|
\le
\frac{\eta}{\delta}(1+2C_\alpha)G
=
\eta M.
\end{equation*}
This completes the proof.
\end{proof}

With these preliminary estimates in hand, we can now state the main convergence result.
\begin{theorem}[Convergence under Lipschitz gradient]
\label{thm:main}
Consider the simplified curvature-aware Adam iteration above. Suppose Assumptions~\ref{assump:smooth}--\ref{assump:oracle} hold, and assume that $F$ is lower bounded by $F_* := \inf_\theta F(\theta) > -\infty$. Let the stepsize satisfy $\eta \le \min(1,\,1/L)$. There exist positive constants $C_1, C_2, C_3$, depending only on $L, \delta, \beta_1, G, C_\alpha, d$, such that for any $T\ge 1$,
\begin{equation}
\frac{1}{T}\sum_{k=0}^{T-1}\mathbb{E}\|\nabla F(\theta_k)\|^2
\le
\frac{C_1\big(F(\theta_0)-F_*+G^2\big)}{\eta T}
+
C_2\eta
+
C_3\sigma^2.
\end{equation}
The last term depends strictly on the stochastic gradient variance, while the number of iterations $T$ can be chosen to reduce the first two terms. Specifically, choosing $\eta=\Theta(T^{-1/2})$ yields a bound of order $\mathcal O(T^{-1/2})+\mathcal O(\sigma^2)$. Thus, in the stochastic setting, the method converges to a neighborhood of stationarity whose radius is controlled by the variance level $\sigma^2$.

\end{theorem}

\begin{proof}
By the $L$-smoothness of $F$, the update rule $\theta_{k+1}-\theta_k=-\eta D_k m_k$, and the intermediate bound $\|D_k m_k\| \le M$ established in Lemma~\ref{lem:bounded-displacement}, we have:
\begin{equation}
\mathbb{E}[F(\theta_{k+1})] \le \mathbb{E}[F(\theta_k)] - \eta\,\mathbb{E}\langle \nabla F(\theta_k),D_k m_k\rangle + \frac{L M^2}{2}\eta^2.
\label{eq:descent-start}
\end{equation}

We define the modified full gradient as $\tilde\nabla_k := \nabla F(\theta_k) + \alpha_k\bigl(\nabla F(\theta_k) - \nabla F(\theta_{k-1})\bigr)$. Let $\Delta_k := \alpha_k\bigl(\nabla F(\theta_k) - \nabla F(\theta_{k-1})\bigr)$ and $e_k := m_k - \tilde\nabla_k$. By expanding $m_k = \nabla F(\theta_k) + \Delta_k + e_k$ and applying the weighted Young's inequality ($\langle a, Db \rangle \le \frac{1}{4}\langle a, Da \rangle + \langle b, Db \rangle$) individually to both the momentum shift $\Delta_k$ and the tracking error $e_k$, together with the preconditioner bounds $c_0 I \preceq D_k \preceq \frac{1}{\delta} I$ in Lemma~\ref{lem:bounded-displacement}, we obtain:
\begin{align}\label{eq:inner_prod_bound}
-\langle \nabla F(\theta_k), D_k m_k \rangle \le -\frac{c_0}{2} \| \nabla F(\theta_k)\|^2 + \frac{1}{\delta}\|\Delta_k\|^2 + \frac{1}{\delta}\| e_k\|^2.
\end{align}

We now bound the expected squared tracking error $\mathbb{E}\|e_k\|^2$. From the momentum recursion, $e_k = \beta_1 e_{k-1} + \beta_1(\tilde\nabla_{k-1}-\tilde\nabla_k) + (1-\beta_1)(\tilde g_k-\tilde\nabla_k)$. Using the convexity of the squared norm and the standard inequality $\|a+b\|^2 \le 2\|a\|^2 + 2\|b\|^2$, we obtain:
\begin{align}
\mathbb{E}\|e_k\|^2 \le \beta_1 \mathbb{E}\|e_{k-1}\|^2 + \frac{2\beta_1^2}{1-\beta_1} \mathbb{E}\|\tilde\nabla_{k-1}-\tilde\nabla_k\|^2 + 2(1-\beta_1) \mathbb{E}\|\tilde g_k-\tilde\nabla_k\|^2.
\end{align}
We bound the last two terms using the $L$-smoothness of $F$, the displacement limit ($\|\theta_k - \theta_{k-1}\| \le \eta M$), and the stochastic variance from Assumption~\ref{assump:oracle}. Defining the constants $c_1 := 3(1 + 2C_\alpha^2) L^2 M^2$ and $c_2 := 3(1 + 2C_\alpha^2)$, we establish:
\begin{align}
\mathbb{E}\|\tilde\nabla_{k-1} - \tilde\nabla_k\|^2 \le c_1 \eta^2 \quad \text{and} \quad \mathbb{E}\|\tilde g_k - \tilde\nabla_k\|^2 \le c_2 \sigma^2.
\end{align}
Substituting these limits back into the recursion yields:
\begin{align}\label{eq:e-contraction}
\mathbb{E}\|e_k\|^2 \le \beta_1\mathbb{E}\|e_{k-1}\|^2 + \frac{2\beta_1^2 c_1}{1-\beta_1}\eta^2 + 2(1-\beta_1)c_2\sigma^2.
\end{align}
With the initialization $m_{-1}=0$, $\theta_{-1}=\theta_0$, and $\alpha_0=0$, we have $\tilde\nabla_0=\nabla F(\theta_0).$ Define the initial tracking error by $e_{-1}:=m_{-1}-\nabla F(\theta_0)=-\nabla F(\theta_0),$ so that $\|e_{-1}\|^2=\|\nabla F(\theta_0)\|^2\le G^2.$
Unrolling the contraction~\eqref{eq:e-contraction} from $k=0$ gives:
\begin{align}\label{eq:e-unrolled}
\mathbb{E}\|e_k\|^2 \le \beta_1^{k+1} G^2 + \frac{2\beta_1^2 c_1}{(1-\beta_1)^2}\eta^2 + 2c_2\sigma^2.
\end{align}

Finally, substituting the inner product bound~\eqref{eq:inner_prod_bound} into~\eqref{eq:descent-start}, and using the tracking-error bound~\eqref{eq:e-unrolled} together with the assumption $\eta\leq 1/L$ (which implies $\eta^3 \le \eta^2 / L$), we obtain:
\begin{align*}
\mathbb{E}[F(\theta_{k+1})]
&\le
\mathbb{E}[F(\theta_k)]
-
\frac{c_0\eta}{2}\mathbb{E}\|\nabla F(\theta_k)\|^2
+
\frac{L M^2}{2}\eta^2
+
\frac{(C_\alpha L M)^2}{\delta}\eta^3
+
\frac{\eta}{\delta}\mathbb{E}\|e_k\|^2 \\
&\le
\mathbb{E}[F(\theta_k)]
-
\frac{c_0\eta}{2}\mathbb{E}\|\nabla F(\theta_k)\|^2
+
\frac{G^2}{\delta}\eta\beta_1^{k+1}
+
c_3\eta^2
+
c_4\eta \sigma^2,
\end{align*}
where the aggregated step-wise constants are defined as:
\begin{align*}
c_3 := \left(\frac{1}{2}+\frac{C_\alpha^2}{\delta}+\frac{6\beta_1^2(1+2C_\alpha^2)}{\delta(1-\beta_1)^2}\right)LM^2, 
\qquad 
c_4 := \frac{6(1+2C_\alpha^2)}{\delta}.
\end{align*}

Summing this inequality over $k=0,\dots,T-1$ and using the global lower bound $F(\theta_T) \ge F_*$ yields:
\begin{equation}
\frac{c_0\eta}{2}\sum_{k=0}^{T-1}\mathbb{E}\|\nabla F(\theta_k)\|^2
\le
F(\theta_0)-F_*
+
\frac{\beta_1 G^2}{\delta(1-\beta_1)}\eta
+
c_3T\eta^2
+
c_4T\eta\sigma^2.
\end{equation}
Dividing both sides by $\frac{c_0\eta T}{2}$ gives
\begin{equation}
\frac{1}{T}\sum_{k=0}^{T-1}\mathbb{E}\|\nabla F(\theta_k)\|^2
\le
\frac{2\big(F(\theta_0)-F_*\big)}{c_0\eta T}
+
\frac{2\beta_1 G^2}{c_0\delta(1-\beta_1)T}
+
\frac{2c_3}{c_0}\eta
+
\frac{2c_4}{c_0}\sigma^2.
\end{equation}
Since $\eta\le 1$, we have $\frac{1}{T}\le \frac{1}{\eta T}$, and therefore
\begin{equation}
\frac{2\big(F(\theta_0)-F_*\big)}{c_0\eta T}
+
\frac{2\beta_1 G^2}{c_0\delta(1-\beta_1)T}
\le
\frac{C_1\big(F(\theta_0)-F_*+G^2\big)}{\eta T},
\end{equation}
for $C_1:=\max\!\left(\frac{2}{c_0},\;\frac{2\beta_1}{c_0\delta(1-\beta_1)}\right).$
Setting $C_2 := \frac{2c_3}{c_0},C_3 := \frac{2c_4}{c_0},$ we arrive at the inequality stated in the theorem. The remaining claims regarding the convergence rate follow immediately from this bound by the stated choices of $\eta$, $T$, and $\sigma^2$.
\end{proof}

We next make the curvature interpretation more explicit under a Lipschitz Hessian assumption. In this regime, the gradient-difference correction becomes a first-order surrogate for a Hessian-vector product, up to a second-order remainder, without changing the convergence guarantee above.

\begin{assumption}[Lipschitz Hessian]
\label{assump:hessian}
There exists $\rho>0$ such that
\begin{equation}
\|\nabla^2 F(x)-\nabla^2 F(y)\| \le \rho \|x-y\|,
\qquad \forall x,y \in \mathbb{R}^d.
\end{equation}
\end{assumption}

\begin{theorem}[Curvature interpretation under Lipschitz Hessian]
\label{thm:hessian}
Under the conditions of Theorem~\ref{thm:main}, suppose that Assumption~\ref{assump:hessian} holds. Then for each $k$,
\begin{equation}
\nabla F(\theta_k)-\nabla F(\theta_{k-1})
=
\left(
\int_0^1 \nabla^2 F\big(\theta_{k-1}+t(\theta_k-\theta_{k-1})\big)\,dt
\right)(\theta_k-\theta_{k-1}),
\end{equation}
and hence
\begin{equation}
\tilde \nabla_k
=
\nabla F(\theta_k)
+
\alpha_k
\left(
\int_0^1 \nabla^2 F\big(\theta_{k-1}+t(\theta_k-\theta_{k-1})\big)\,dt
\right)(\theta_k-\theta_{k-1}).
\end{equation}
Moreover,
\begin{equation}
\|\tilde \nabla_k - \nabla F(\theta_k) - \alpha_k \nabla^2 F(\theta_k)(\theta_k-\theta_{k-1})\|
\le
C_\alpha \rho \|\theta_k-\theta_{k-1}\|^2.
\end{equation}
Therefore, the correction term is a first-order surrogate of a Hessian-vector product, up to a second-order remainder, and the convergence guarantee in Theorem~\ref{thm:main} remains valid.
\end{theorem}

\begin{proof}
Let $s_k := \theta_k-\theta_{k-1}$. By the integral form of Taylor's theorem applied to $\nabla F$,
\begin{equation}
\nabla F(\theta_k)-\nabla F(\theta_{k-1})
=
\left(
\int_0^1 \nabla^2 F(\theta_{k-1}+t s_k)\,dt
\right)s_k,
\end{equation}
which immediately gives the stated representation of $\tilde \nabla_k$.

Subtracting $\alpha_k \nabla^2 F(\theta_k)s_k$ and using Assumption~\ref{assump:hessian}, we obtain
\begin{align}
\|\tilde \nabla_k - \nabla F(\theta_k) - \alpha_k \nabla^2 F(\theta_k)s_k\|
&\le
|\alpha_k|
\int_0^1
\|\nabla^2 F(\theta_{k-1}+t s_k)-\nabla^2 F(\theta_k)\|\,dt\,\|s_k\| \\
&\le
C_\alpha \rho
\int_0^1
\|\theta_{k-1}+t s_k-\theta_k\|\,dt\,\|s_k\| \\
&=
C_\alpha \rho
\int_0^1
(1-t)\,dt\,\|s_k\|^2
\le
C_\alpha \rho \|s_k\|^2.
\end{align}
This proves the claimed approximation bound.

Finally, since $\|s_k\|=O(\eta)$, the additional term is of order $O(\eta^2)$ and can be absorbed into the higher-order terms in the proof of Theorem~\ref{thm:main} without affecting the convergence rate.
\end{proof}

\begin{remark}
Theorem~\ref{thm:hessian} justifies the term ``curvature-aware'': At the full gradient level, the correction
\begin{equation}
\alpha_k\big(\nabla F(\theta_k)-\nabla F(\theta_{k-1})\big)
\end{equation}
is exactly a Hessian-averaged vector product along the segment from $\theta_{k-1}$ to $\theta_k$, and therefore serves as a first-order surrogate for
\begin{equation}
\alpha_k \nabla^2 F(\theta_k)(\theta_k-\theta_{k-1}),
\end{equation}
up to a second-order remainder under Lipschitz Hessian continuity.
\end{remark}

\section{Numerical Experiments}
\label{Section: Experiment}
In this section, we evaluate the proposed CA framework on a diverse set of PDE benchmarks. We combine it with three widely used optimizers in deep learning and PINNs, namely AdamW~\cite{kingma2014adam,loshchilov2018decoupled}, SOAP~\cite{vyas2025soap}, and Muon~\cite{jordan2024muon}, yielding CA-AdamW, CA-SOAP, and CA-Muon, respectively.

For each problem setting, the testing set contains 90{,}000 randomly sampled points and is kept fixed across all experiments for fair comparison. Training samples are independently drawn from the corresponding domains, where $N_f$, $N_b$, and $N_0$ denote the numbers of residual, boundary, and initial points, respectively, when applicable. For each PDE, all models are trained and evaluated with five random seeds. To isolate the effect of the CA modification, we use the same learning-rate schedule and weight decay for the standard and CA variants within each optimizer family. Performance is measured by the relative $L^2$ error and the $L^\infty$ norm:
\begin{align*}
\text{Relative } L^2 \text{ error}
&=
\frac{\sqrt{\sum_{k=1}^{N}\left|\hat{u}(\mathbf{x}_k,t_k)-u(\mathbf{x}_k,t_k)\right|^2}}
{\sqrt{\sum_{k=1}^{N}\left|u(\mathbf{x}_k,t_k)\right|^2}},\\
L^{\infty} \text{ norm}
&=
\max_{1 \leq k \leq N}
\left|\hat{u}(\mathbf{x}_k,t_k)-u(\mathbf{x}_k,t_k)\right|.
\end{align*}
Here, $u$ denotes the exact solution, $\hat{u}$ the model prediction, and $N$ the number of testing samples. We report the best result over five runs in the figures and the average over five runs in the tables.

\subsection{10D Heat Equation}
\label{Section: 10D Heat}
We consider the following high dimensional heat equation on the domain $\Omega = [-1,1]^{d}$ and $T = [0,1]$:
\begin{align}
    &u_t - \Delta u = f(\mathbf{x},t), \quad &&(\mathbf{x},t)\in\Omega\times T,\\
    &u(\mathbf{x},t) = g(\mathbf{x},t), \quad &&(\mathbf{x},t)\in\partial\Omega\times T,\\
    &u(\mathbf{x},0) = h(\mathbf{x}), \quad &&\mathbf{x}\in\Omega,
\end{align}
where $f(\mathbf{x},t)$, $g(\mathbf{x},t)$, and $h(\mathbf{x})$ can be inferred from the analytical solution $u(\mathbf{x},t) = \cos(\frac{1}{d}\sum_{i=1}^dx_i)e^{-t}$, where $\mathbf{x} = [x_1,x_2,\ldots,x_d]$, and we choose $d = 10$ in this example. Also, we will use a 3-layer network with 80 neurons in each hidden layer.

For training, we sample $N_f = 3{,}000$ collocation points in the domain and set $N_b = N_0 = 100$ for the boundary and initial data. The loss weights are chosen as $\lambda_F = 1$ and $\lambda_B = \lambda_I = 10$. The relative $L^2$ error histories are shown in Fig.~\ref{fig: Heat Hist}, and the final numerical results are summarized in Table~\ref{tab: Heat table}.

As shown in Fig.~\ref{fig: Heat Hist} and Table~\ref{tab: Heat table}, the CA modification improves final accuracy and often accelerates convergence across tested optimizers. Standard AdamW plateaus at a relative $L^2$ error of about $5.66 \times 10^{-3}$, whereas CA-AdamW reduces it to $1.34 \times 10^{-4}$, corresponding to a relative error ratio of $1.87\%$. This indicates that CA effectively alleviates optimization stagnation in the 10D setting.

Among all methods, CA-SOAP achieves the best accuracy, with the lowest relative $L^2$ error of $5.98 \times 10^{-6}$ and $L^\infty$ error of $1.21 \times 10^{-4}$. Although Muon and SOAP outperform AdamW in their standard forms, their CA variants further reduce the errors substantially. Overall, these results show that the CA modification is a robust enhancement for improving convergence and numerical accuracy on the heat equation.

\begin{figure}[!htb]
\centering
    \begin{minipage}{0.5\textwidth}
     \centering
     \includegraphics[width=\linewidth]{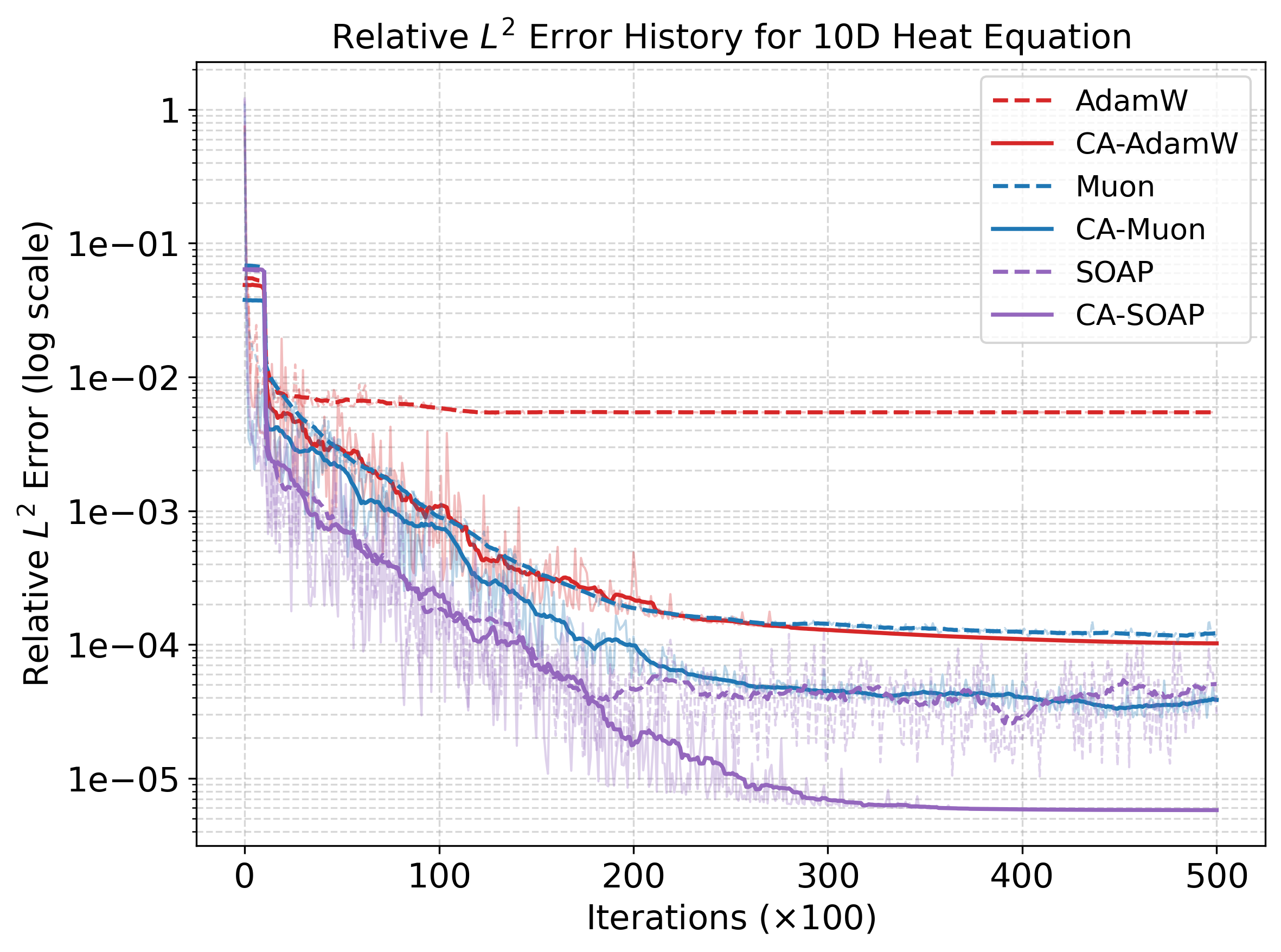} 
   \end{minipage}   
    \caption{History of the relative $L^2$ errors for the 10D Heat Equation.}\label{fig: Heat Hist}
\end{figure}

\begin{table}[!htb]
    \centering
    \caption{\textbf{Errors for the 10D heat equation (lower is better).} Average relative $L^2$ error and absolute $L^\infty$ error for each optimizer, comparing the standard version to its CA variant. \textbf{Error reduction (\%)} indicates the percentage decrease in error from the standard optimizer to its CA variant (\textbf{higher is better}). \gbf{Green} denotes the best results and performance improvements across all optimizers.}
    \label{tab: Heat table}
    \resizebox{0.6\textwidth}{!}{
        \begin{tabular}{lccc}
        \toprule
        \textbf{Optimizers}
        & \multicolumn{1}{c}{Rel. $L^2$ error ($u$)} & \multicolumn{1}{c}{$L^\infty$ error ($u$)} \\
        \hline

        \grow AdamW 
        & $5.66\times 10^{-3}$ & $6.71\times 10^{-2}$ \\
        
        \brow CA-AdamW
        & $1.34\times 10^{-4}$ & $1.07\times 10^{-3}$ \\
        
        \multicolumn{1}{l|}{\textbf{Error reduction (\%) $\uparrow$}}
            & \gbf{97.63\%} & \gbf{98.40\%} \\
        \hline
        
        \grow Muon
        & $1.41\times 10^{-4}$ & $7.94\times 10^{-3}$ \\
        
        \brow CA-Muon
        & $3.92\times 10^{-5}$ & $2.18\times 10^{-3}$ \\
        
        \multicolumn{1}{l|}{\textbf{Error reduction (\%) $\uparrow$}}
            & 72.20\% & 72.54\% \\
        \hline
        
        \grow SOAP
        & $5.84 \times 10^{-5}$ & $8.83\times 10^{-4}$ \\
        
        \brow CA-SOAP
        & \gbf{$\mathbf{5.98\times 10^{-6}}$} & \gbf{$\mathbf{1.21\times 10^{-4}}$} \\
        
        \multicolumn{1}{l|}{\textbf{Error reduction (\%) $\uparrow$}}
            & 89.76\% & 86.30\% \\
        
        \bottomrule
        \end{tabular}
    }
    \vspace{-0.5em}
\end{table}

\subsection{Gray-Scott Equation}
\label{Section: GS equation}
Reaction--diffusion systems generate diverse patterns observed in nature. The Gray--Scott system is a canonical model of such dynamics and a challenging benchmark for PINN training due to its stiff reaction--diffusion behavior, nonlinear coupling, and sensitivity to error accumulation over long time horizons. It is governed by the following equations:
\begin{align}
    &u_t - \epsilon_1 u_{xx} - b(1-u) + uv^2 = 0, \quad &&(x, t)\in\Omega\times T,\\
    &v_t - \epsilon_2 v_{xx} + (b + k)v + uv^2 = 0, \quad &&(x, t)\in\Omega\times T,\\
    &u(x,0) = u_0(x), \quad v(x,0) = v_0(x), \quad &&x\in\Omega,
\end{align}
with a periodic boundary condition. The spatial domain is defined as $\Omega = [-50,50]$ with a time horizon of $T = [0,20]$. The diffusion rates are set to $\epsilon_1 = 1$ and $\epsilon_2 = 0.01$, while the reaction coefficients are $b=0.02$ and $k=0.0562$. The initial conditions are chosen to be:
\begin{align}
    u_0(x) = 1 - \frac{\sin(\pi(x-50)/100)^4}{2},\quad v_0(x) = \frac{\sin(\pi(x-50)/100)^4}{4}.
\end{align}

We employ a fully connected neural network architecture with depth 3 and width 128. To strictly enforce the periodic boundary condition as a hard constraint, we utilize a Fourier feature embedding layer~\cite{dong2021method,anagnostopoulos2024residual} that maps the spatial input $x$ to 
\begin{align*}
    [\cos(\omega_x x), \sin(\omega_x x), \dots, \cos(m\omega_x x), \sin(m\omega_x x)],
\end{align*} 
where $\omega_x = 2\pi/100$ and $m=10$. This embedding ensures that the network output is inherently periodic with period $L=100$. The dataset consists of $N_f = 5{,}000$ collocation points and $N_0 = 100$ initial points. The loss weights are set to $\lambda_F = 1$ and $ \lambda_I = 100$.

The training configuration, including the learning rate schedule, and weight decay remains identical to the setup described in the previous section. However, due to the chaotic and stiff nature of the Gray-Scott system, we employ a time-marching strategy to enhance convergence~\cite{krishnapriyan2021characterizing, vyas2025soap}.

The time domain $T=[0,20]$ is uniformly divided into 10 windows, and the model is trained sequentially over them. For the first window, the initial condition is given by the exact problem setup; for each subsequent window, it is taken from the prediction at the end of the previous window. Although this incurs additional retraining cost, it simplifies learning by restricting the network to shorter temporal evolutions. Each window is trained for $20{,}000$ iterations.

Fig.~\ref{fig: GS Hist} shows the relative $L^2$ error history, and Table~\ref{tab: GS table} reports the global errors over the full spatiotemporal domain. The periodic spikes in Fig.~\ref{fig: GS Hist} are caused by the time-marching scheme: each new time window introduces a transient re-initialization difficulty before reconvergence. Therefore, the per-window error trajectories in the figure differ from the cumulative global errors in the table.

Table~\ref{tab: GS table} shows three main trends. First, the CA framework consistently improves all three optimizer families, indicating that its benefit is not optimizer-specific. Second, the gains are particularly large for AdamW: CA reduces the relative $L^2$ error of $u$ to $5.69\%$ of that of standard AdamW. Third, CA also improves strong baselines such as Muon and SOAP in both relative $L^2$ and $L^\infty$ errors, suggesting that its effect is complementary rather than redundant.

To complement the quantitative results, Fig.~\ref{fig: GS heatmap} presents spatiotemporal heatmaps for CA-SOAP, one of the best-performing models in Table~\ref{tab: GS table}. It accurately captures the evolution of both $u$ and $v$, with only minor local discrepancies in the error maps.

\begin{table}[htbp]
    \footnotesize
    \centering
    \caption{\textbf{Errors for the Gray-Scott System (lower is better).} Average relative $L^2$ error and absolute $L^\infty$ error of $u$ and $v$ for each optimizer, comparing the standard version to its CA variant.}
    \label{tab: GS table}
    
    \resizebox{0.85\textwidth}{!}{
        \begin{tabular}{lcccc}
        \toprule
        \textbf{Optimizers}
        & \multicolumn{2}{c}{Rel. $L^2$ error} 
        & \multicolumn{2}{c}{$L^\infty$ error} \\
        \cmidrule(lr){2-3} \cmidrule(lr){4-5}
        & $u$ & $v$ & $u$ & $v$ \\
        \hline

        \grow AdamW 
        & $1.44\times 10^{-1}$ & $3.35\times 10^{-1}$ & $2.99\times 10^{-1}$ & $2.08\times 10^{-1}$ \\
        
        \brow CA-AdamW
        & $8.20\times 10^{-3}$ & $4.85\times 10^{-2}$ & $3.48\times 10^{-2}$ & $5.56\times 10^{-2}$ \\
        
        \multicolumn{1}{l|}{\textbf{Error reduction (\%) $\uparrow$}}
            & \gbf{94.31\%} & 85.52\% & \gbf{88.36\%} & 73.26\% \\
        \hline
        
        \grow Muon
        & $3.99\times 10^{-4}$ & $3.82\times 10^{-3}$ & $1.31\times 10^{-3}$ & $2.42\times 10^{-3}$ \\
        
        \brow CA-Muon
        & \gbf{$\mathbf{5.66\times 10^{-5}}$} & $7.85\times 10^{-4}$ & \gbf{$\mathbf{3.50\times 10^{-4}}$} & \gbf{$\mathbf{7.81\times 10^{-4}}$} \\
        
        \multicolumn{1}{l|}{\textbf{Error reduction (\%) $\uparrow$}} 
            & 85.81\% & 79.45\% & 73.28\% & 67.72\% \\
        \hline
        
        \grow SOAP
         & $4.51\times 10^{-4}$ & $5.05\times 10^{-3}$ & $3.44\times 10^{-3}$ & $5.21\times 10^{-3}$ \\
        
        \brow CA-SOAP
        & $6.88\times 10^{-5}$ & \gbf{$\mathbf{6.57\times 10^{-4}}$} & $5.30\times 10^{-4}$ & $9.98\times 10^{-4}$ \\
        
        \multicolumn{1}{l|}{\textbf{Error reduction (\%) $\uparrow$}}
            & 84.75\% & \gbf{86.99\%} & 84.59\% & \gbf{80.84\%} \\
        
        \bottomrule
        \end{tabular}
    }
    \vspace{-0.5em}
\end{table}

\begin{figure}[!htb]
\centering
   \begin{minipage}{\textwidth}
     \centering
     \includegraphics[width=\linewidth]{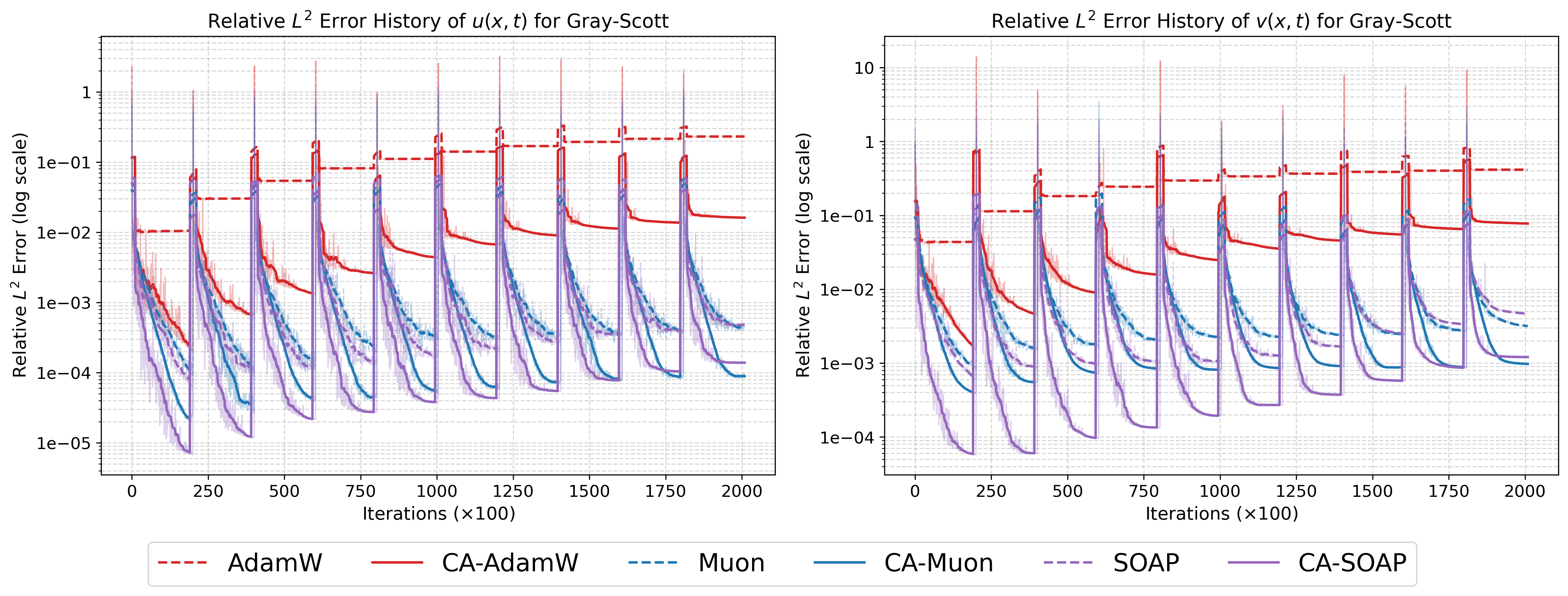} 
   \end{minipage}   
    \caption{History of the relative $L^2$ errors for the Gray-Scott system. Because training is performed using a time-marching strategy, the error curves exhibit periodic spikes at the transition between consecutive time windows (every $20{,}000$ iterations), where the model must adapt to a newly initialized subproblem. Within each window, the error decreases as optimization progresses. 
    }\label{fig: GS Hist}
\end{figure}

\begin{figure}[!htb]
\centering
   \begin{minipage}{0.3\textwidth}
     \centering
     \includegraphics[width=\linewidth]{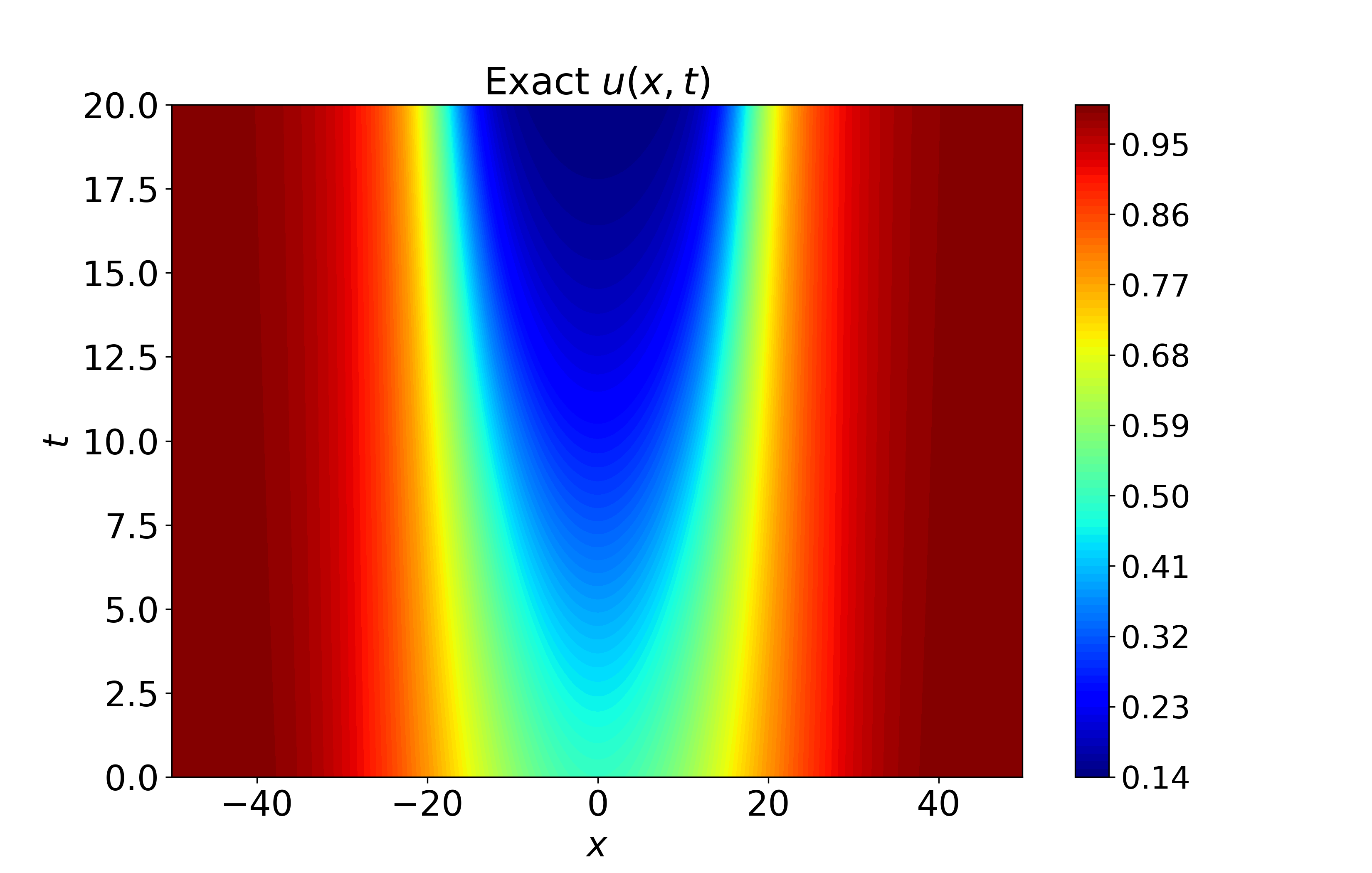} 
   \end{minipage}  
   \begin{minipage}{0.3\textwidth}
     \centering
     \includegraphics[width=\linewidth]{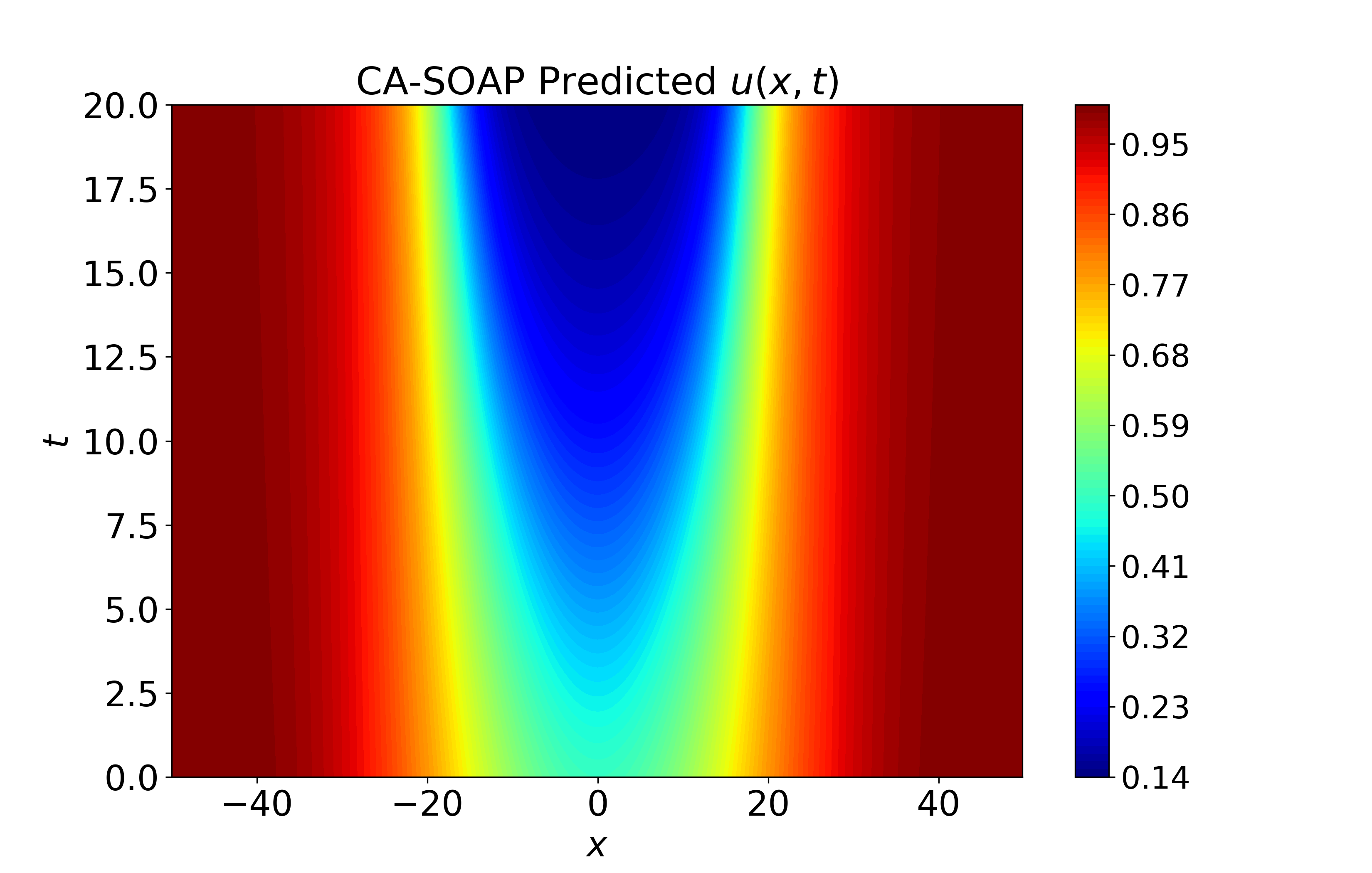} 
   \end{minipage} 
   \begin{minipage}{0.3\textwidth}
     \centering
     \includegraphics[width=\linewidth]{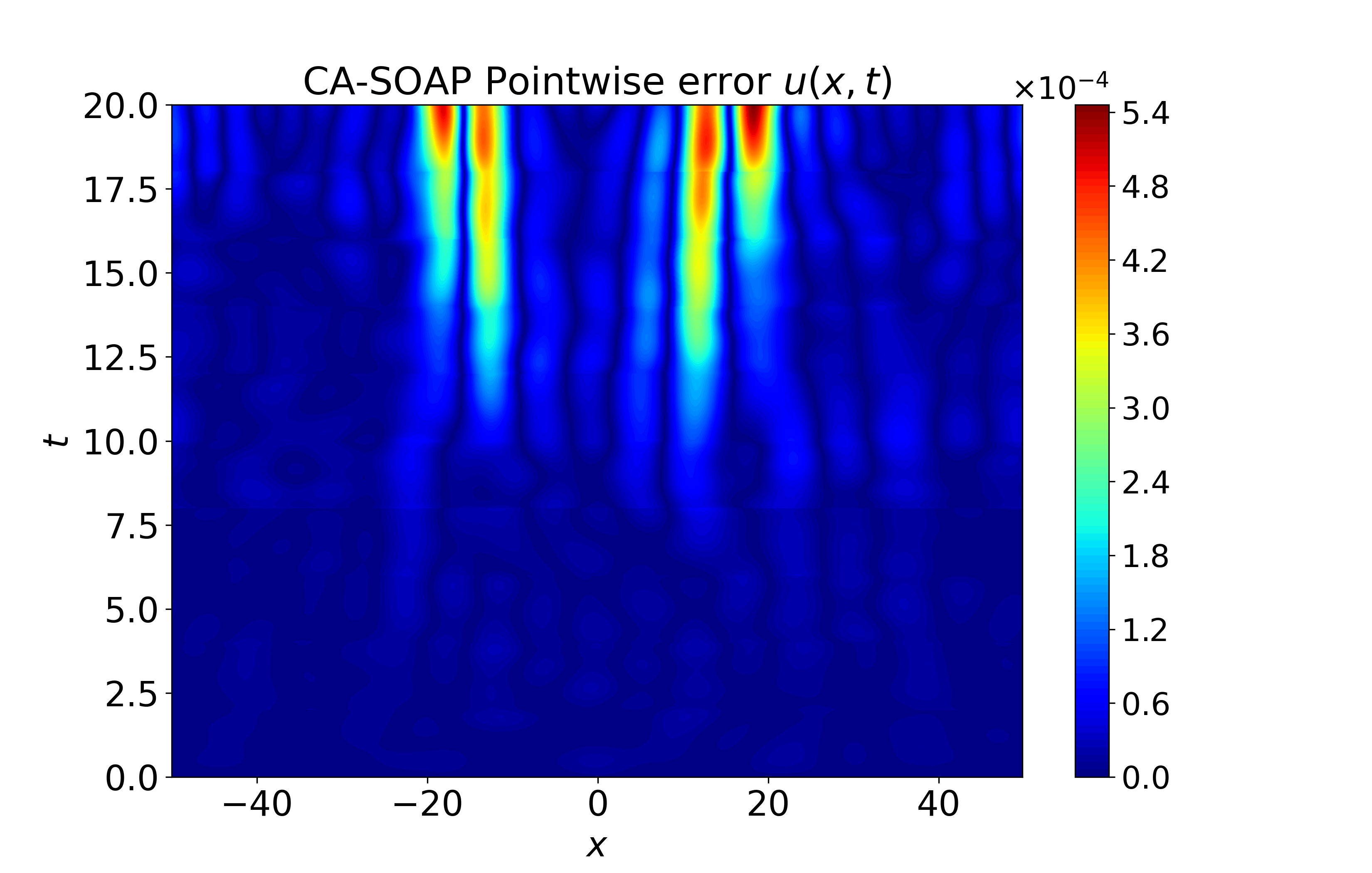} 
   \end{minipage} 

   \begin{minipage}{0.3\textwidth}
     \centering
     \includegraphics[width=\linewidth]{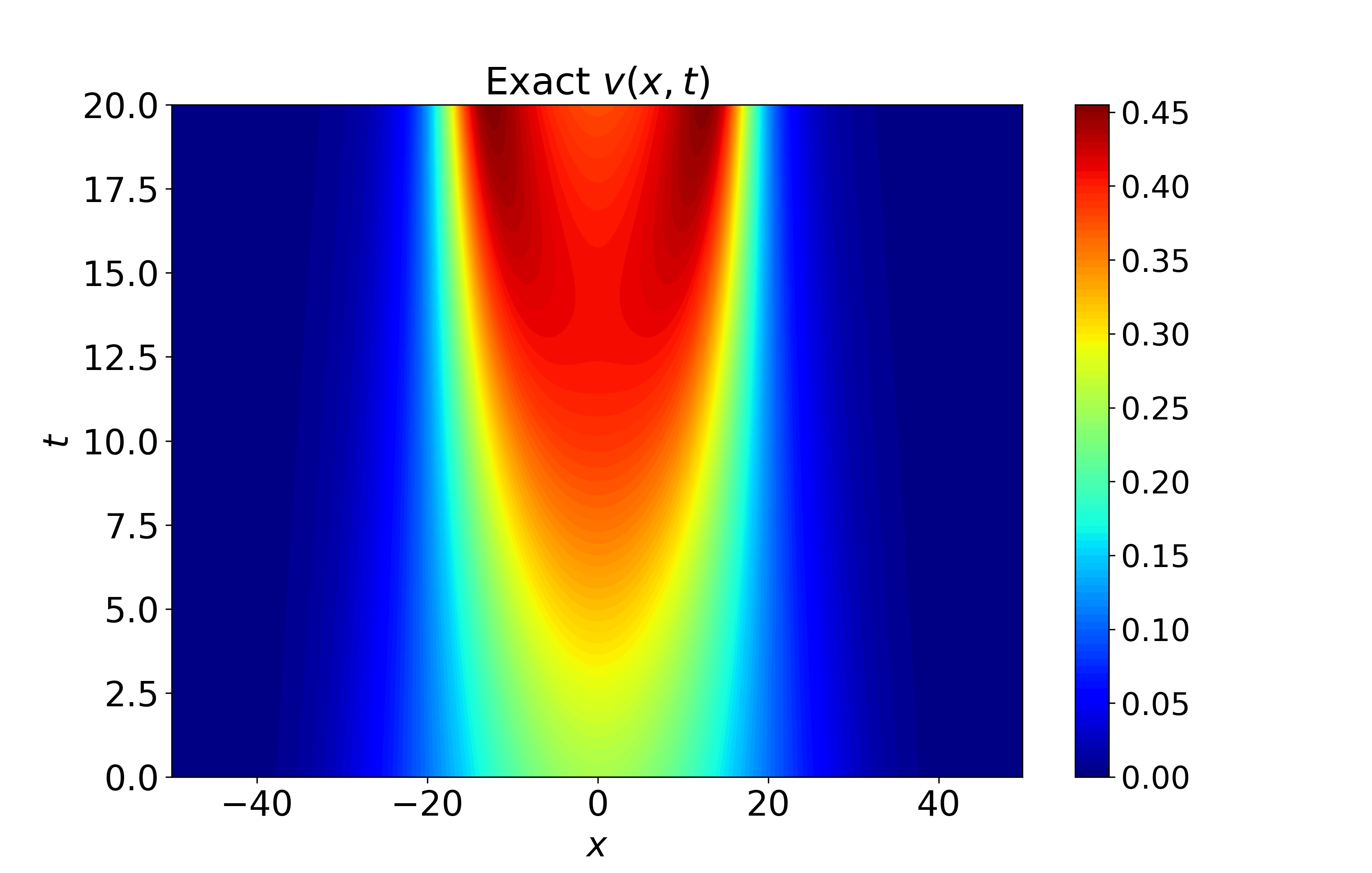} 
   \end{minipage}  
   \begin{minipage}{0.3\textwidth}
     \centering
     \includegraphics[width=\linewidth]{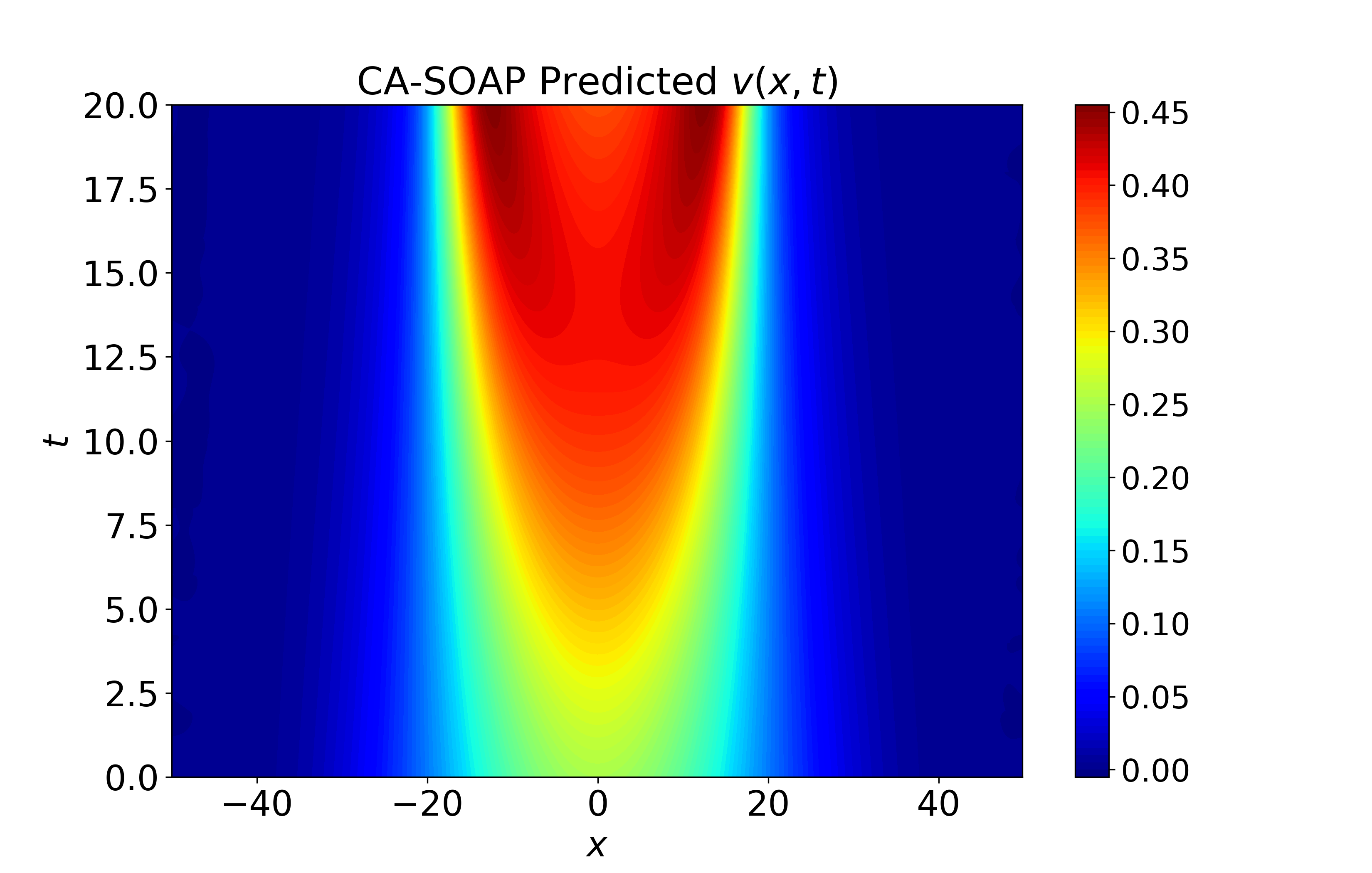} 
   \end{minipage} 
   \begin{minipage}{0.3\textwidth}
     \centering
     \includegraphics[width=\linewidth]{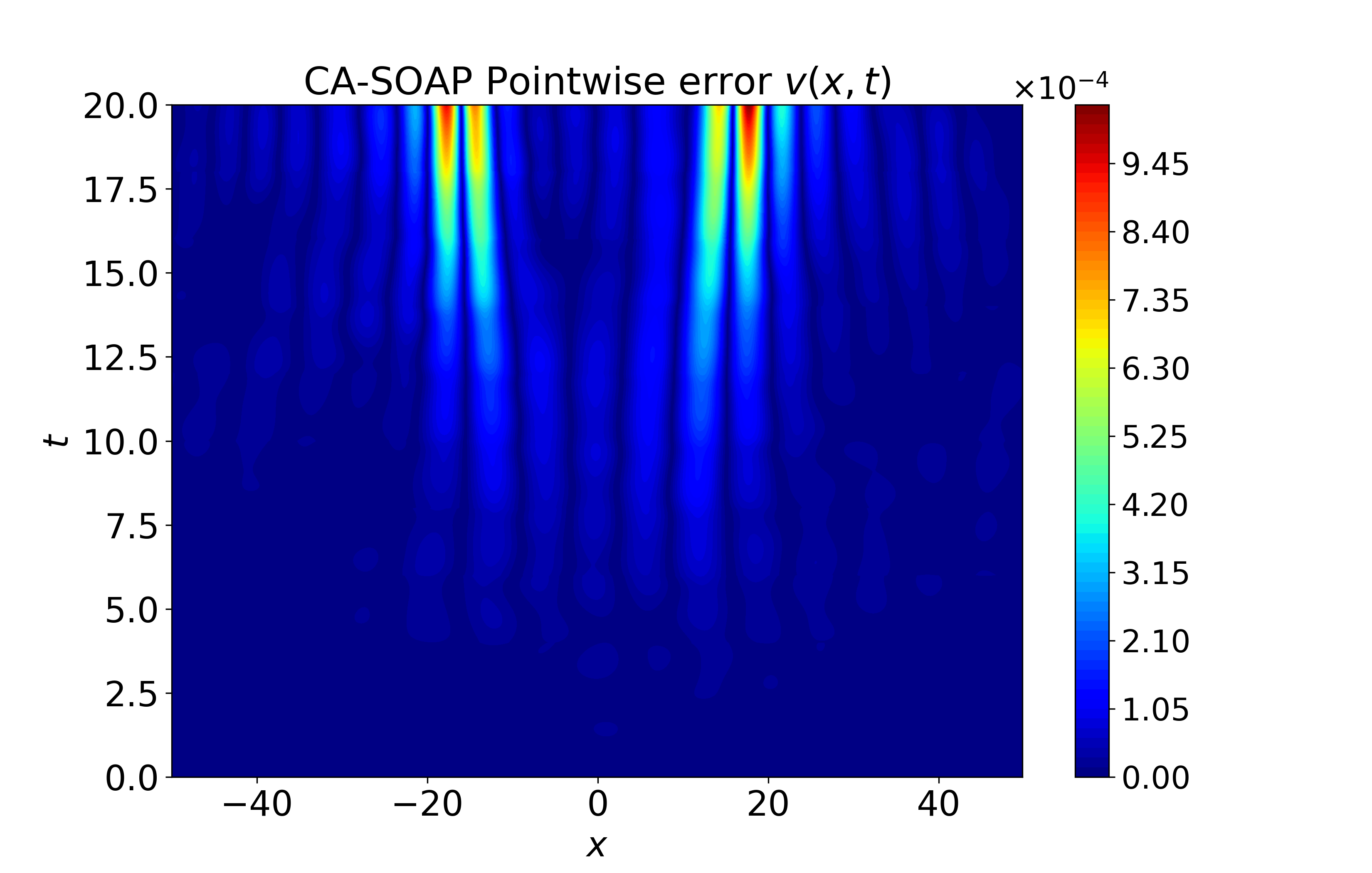} 
   \end{minipage} 
   
    \caption{Spatiotemporal heatmaps for the Gray-Scott system using the best-performing optimizer CA-SOAP. For each variable ($u$ and $v$), we show the reference solution, the prediction, and the corresponding point-wise error over the full spatiotemporal domain. The predicted fields closely match the reference patterns, with errors remaining small throughout the time horizon.}\label{fig: GS heatmap}
\end{figure}

\subsection{Belousov-Zhabotinsky Equation}
\label{Section: BZ equation}
The Belousov--Zhabotinsky (BZ) reaction is a classic nonequilibrium chemical system involving bromine-mediated oxidation in acidic solution. We consider the following coupled BZ system~\cite{rathore2024challenges}:
\begin{align}
    &u_t - \epsilon_1 u_{xx} + u + v - uv - u^2 = 0, \quad &&(x, t)\in\Omega\times T,\\
    &v_t - \epsilon_2 v_{xx} + w - v - uv = 0, \quad &&(x, t)\in\Omega\times T,\\
    &w_t - \epsilon_1 w_{xx} + u - w = 0, \quad &&(x, t)\in\Omega\times T,\\
    &u(x,0) = u_0(x), \quad v(x,0) = v_0(x), \quad w(x,0) = w_0(x), \quad &&x\in\Omega,
\end{align}
subject to periodic boundary conditions. The spatial domain is defined as $\Omega = [-1,1]$ with a time horizon of $T = [0,3]$. The diffusion rates are set to $\epsilon_1 = 10^{-5}$ and $\epsilon_2 = 2\times 10^{-5}$. The initial conditions are chosen to be:
\begin{align*}
    u_0(x) = \exp(-100(x+0.5)^2), v_0(x) = \exp(-100x^2), w_0(x) = \exp(-100(x-0.5)^2).
\end{align*}

We utilize a fully connected neural network architecture with depth 5 and width 128. Also, we employ $N_f = 8{,}000$ collocation points, and $N_0 = 800$ initial points for training, with the weight $\lambda_F = 1$ and $\lambda_I = 100$. Same as the Gray-Scott system in Section~\ref{Section: GS equation}, the periodic boundary condition is enforced as a hard constraint by the Fourier feature embedding by setting $L = 2$ and $m = 10$.

To address the stiff dynamics introduced by the small diffusion coefficients, we again apply the time-marching strategy by dividing the time domain $T = [0, 3]$ into 10 uniform windows. The model will be trained for $20{,}000$ iterations per window.

Fig. \ref{fig: BZ Hist} illustrates the history of the relative $L^2$ errors for the species $u$, $v$, and $w$. While the error curves display the expected re-initialization spikes characteristic of the time-marching scheme, the optimizers successfully drive the error down within each interval, maintaining stability throughout the simulation.

Table \ref{tab: BZ table} quantifies these robust improvements across the tested settings. The addition of the CA strategy consistently enhances prediction accuracy across all tested optimizers. For instance, CA-AdamW improves upon the standard AdamW results by nearly an order of magnitude. Among the tested methods, CA-SOAP achieves the best overall performance and attains the lowest error in most reported metrics, while CA-Muon is also highly competitive on selected quantities (reaching the order of $10^{-5}$). The CA strategy is highly effective at suppressing error accumulation, reducing the final error to a small fraction of the baseline. These results confirm that the CA mechanism serves as an effective stabilizing mechanism in these experiments, enabling every optimizer to handle the stiff BZ dynamics with superior fidelity. Moreover, Fig.~\ref{fig: BZ heatmap} presents the spatiotemporal heatmaps of the best-performing model, CA-SOAP, for $u$, $v$, and $w$. The predicted fields agree closely with the reference solutions, and the corresponding error maps remain small throughout the domain.

\begin{figure}[!htb]
\centering
   \begin{minipage}{0.3\textwidth}
     \centering
     \includegraphics[width=\linewidth]{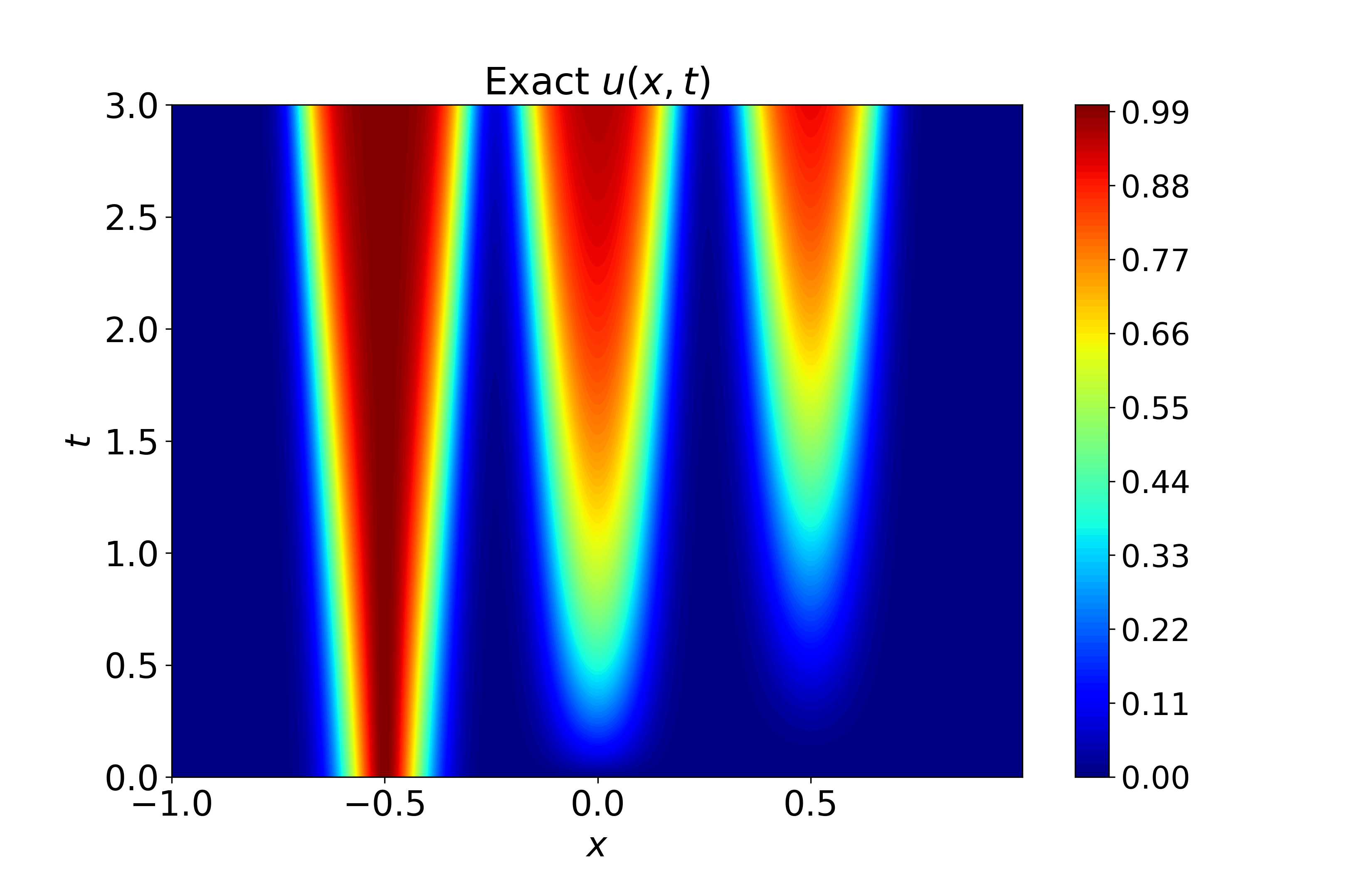} 
   \end{minipage}  
   \begin{minipage}{0.3\textwidth}
     \centering
     \includegraphics[width=\linewidth]{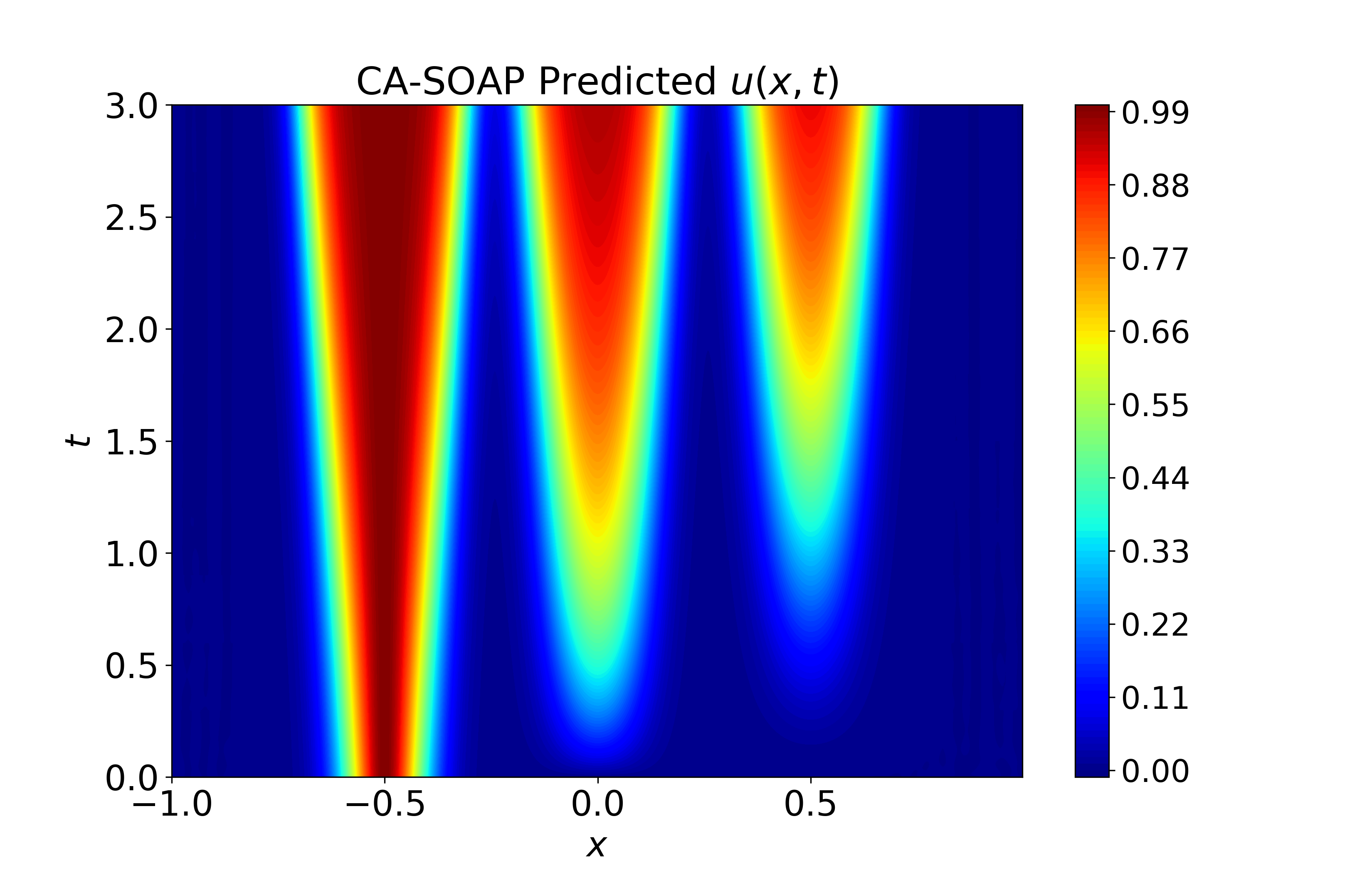} 
   \end{minipage} 
   \begin{minipage}{0.3\textwidth}
     \centering
     \includegraphics[width=\linewidth]{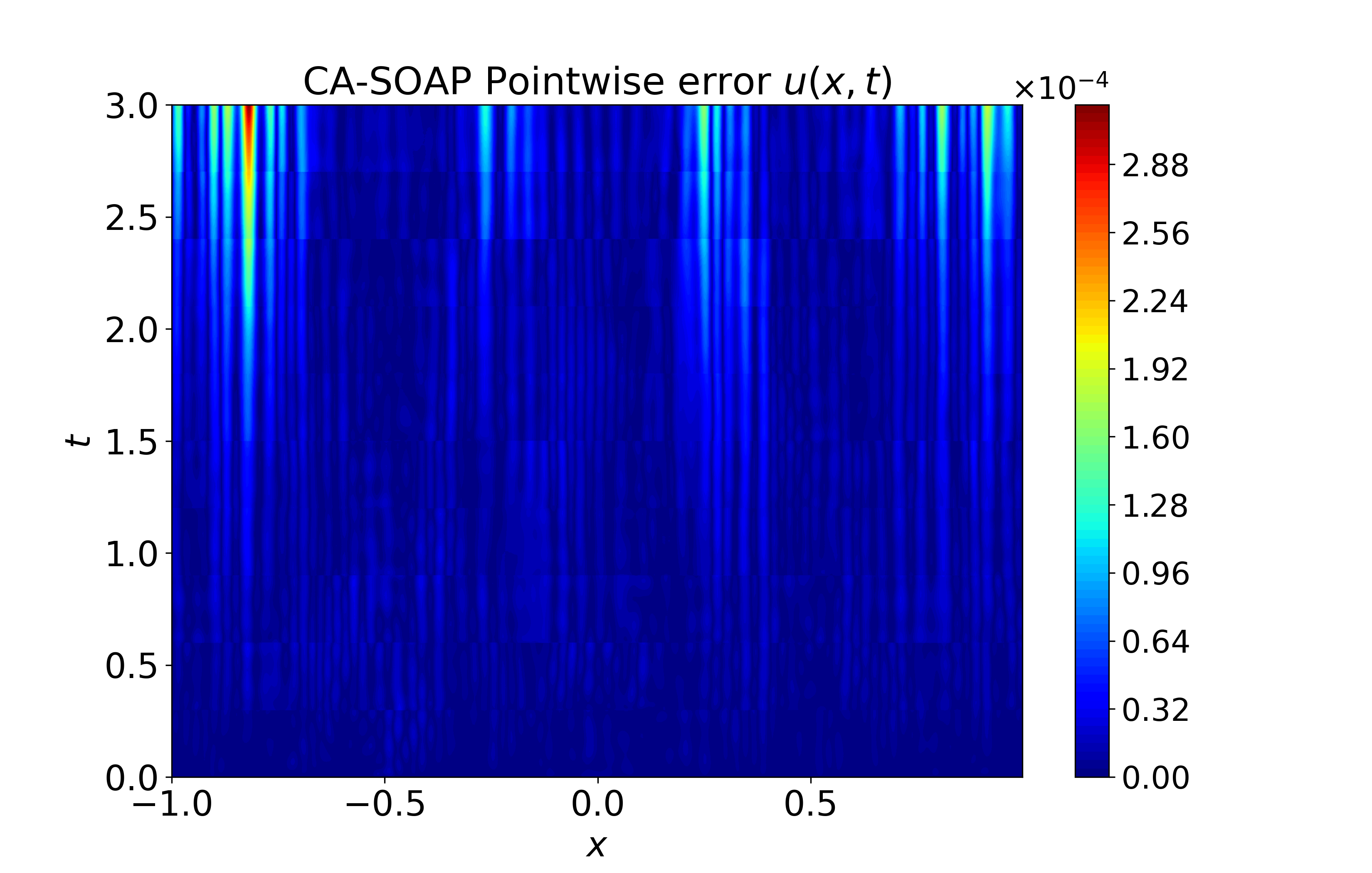} 
   \end{minipage} 

   \begin{minipage}{0.3\textwidth}
     \centering
     \includegraphics[width=\linewidth]{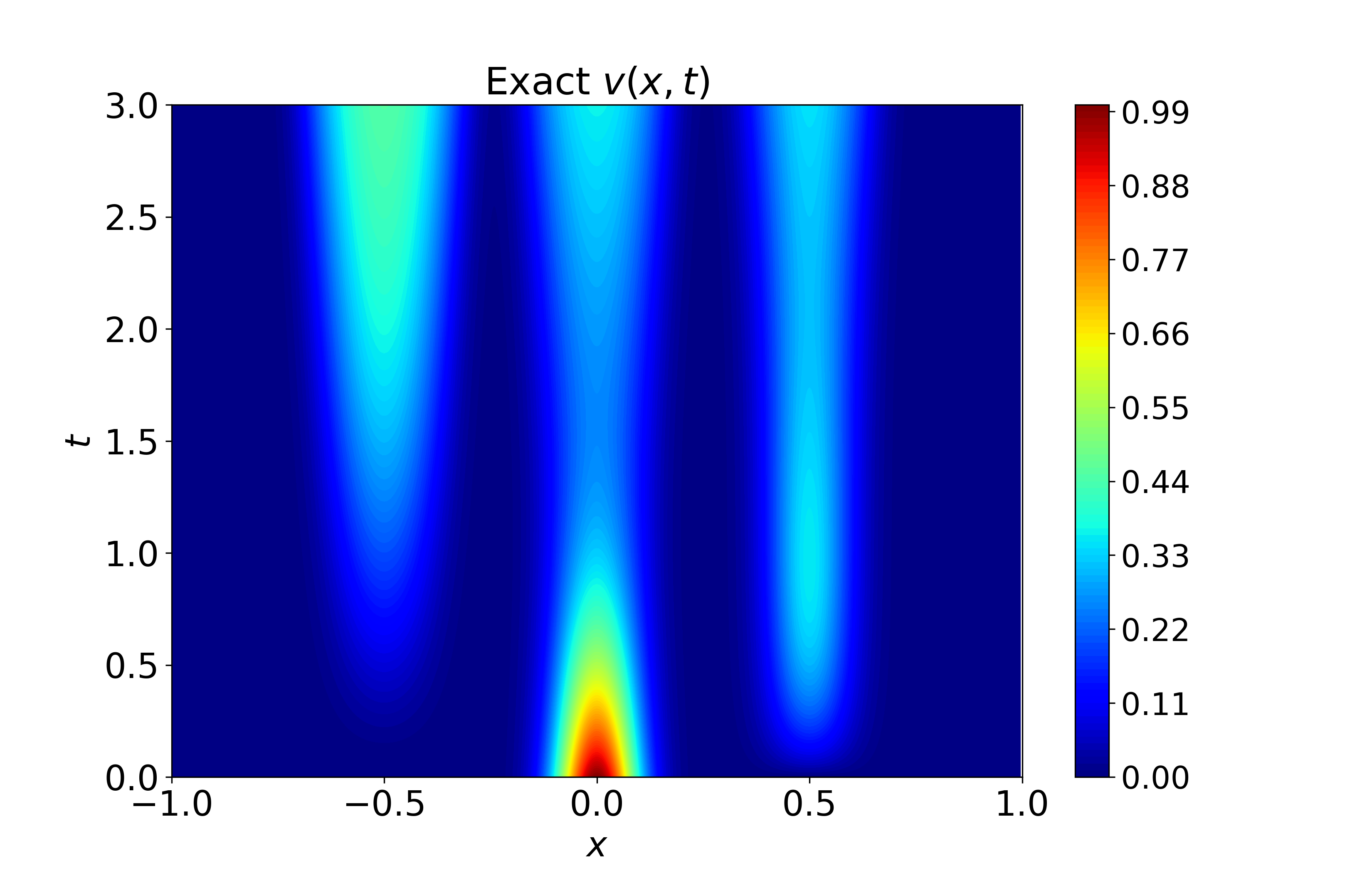} 
   \end{minipage}  
   \begin{minipage}{0.3\textwidth}
     \centering
     \includegraphics[width=\linewidth]{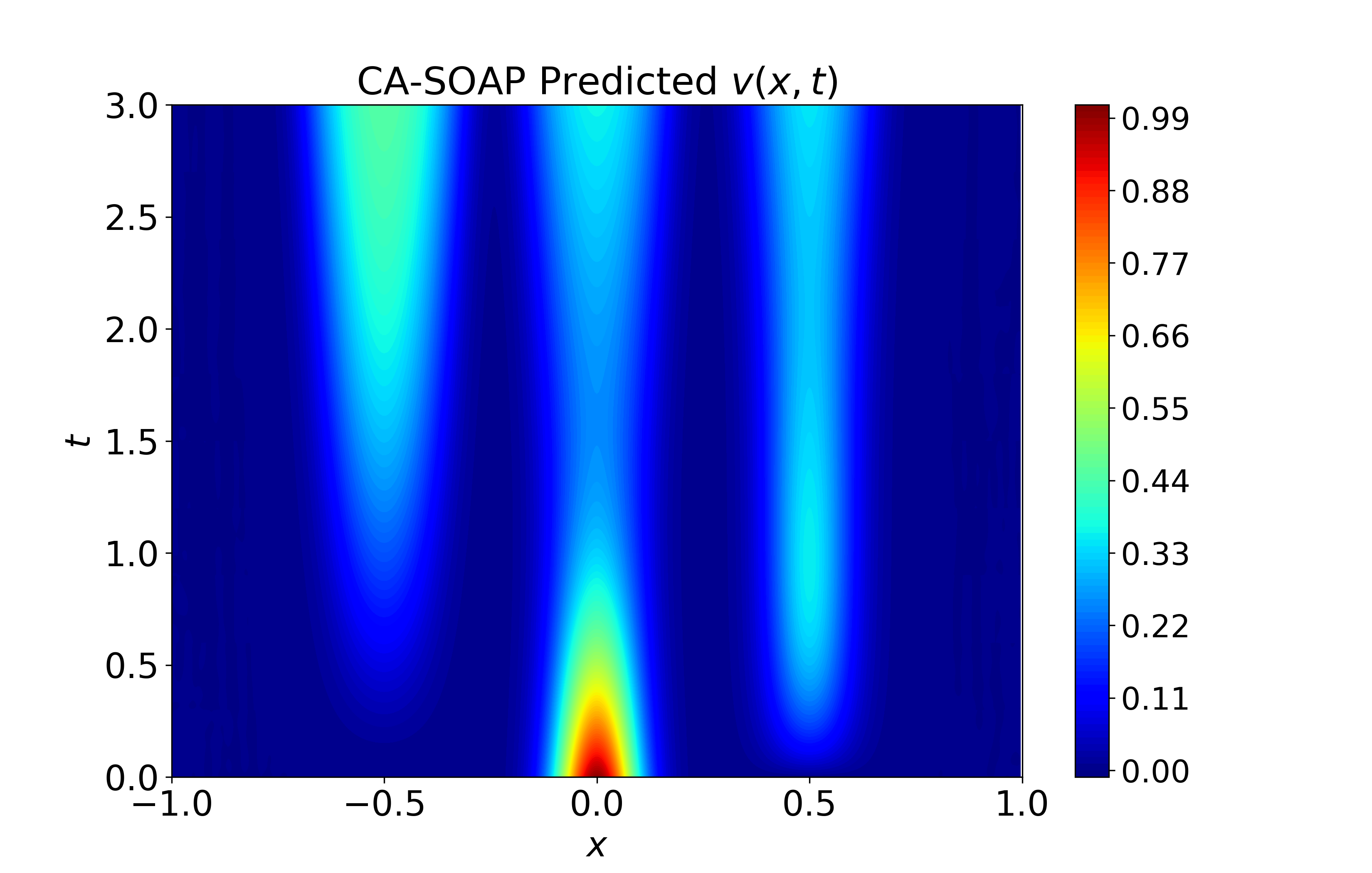} 
   \end{minipage} 
   \begin{minipage}{0.3\textwidth}
     \centering
     \includegraphics[width=\linewidth]{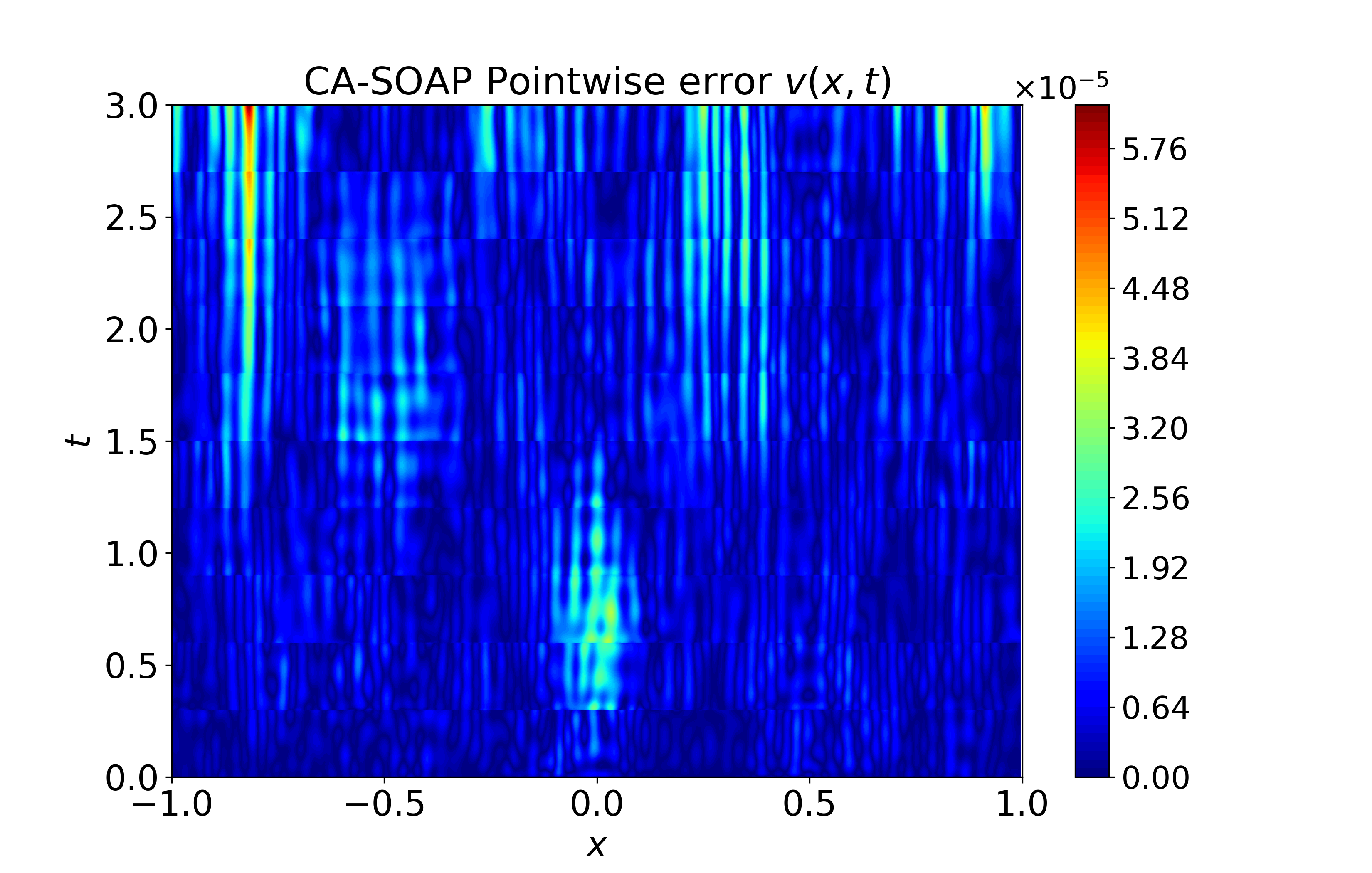} 
   \end{minipage} 

      \begin{minipage}{0.3\textwidth}
     \centering
     \includegraphics[width=\linewidth]{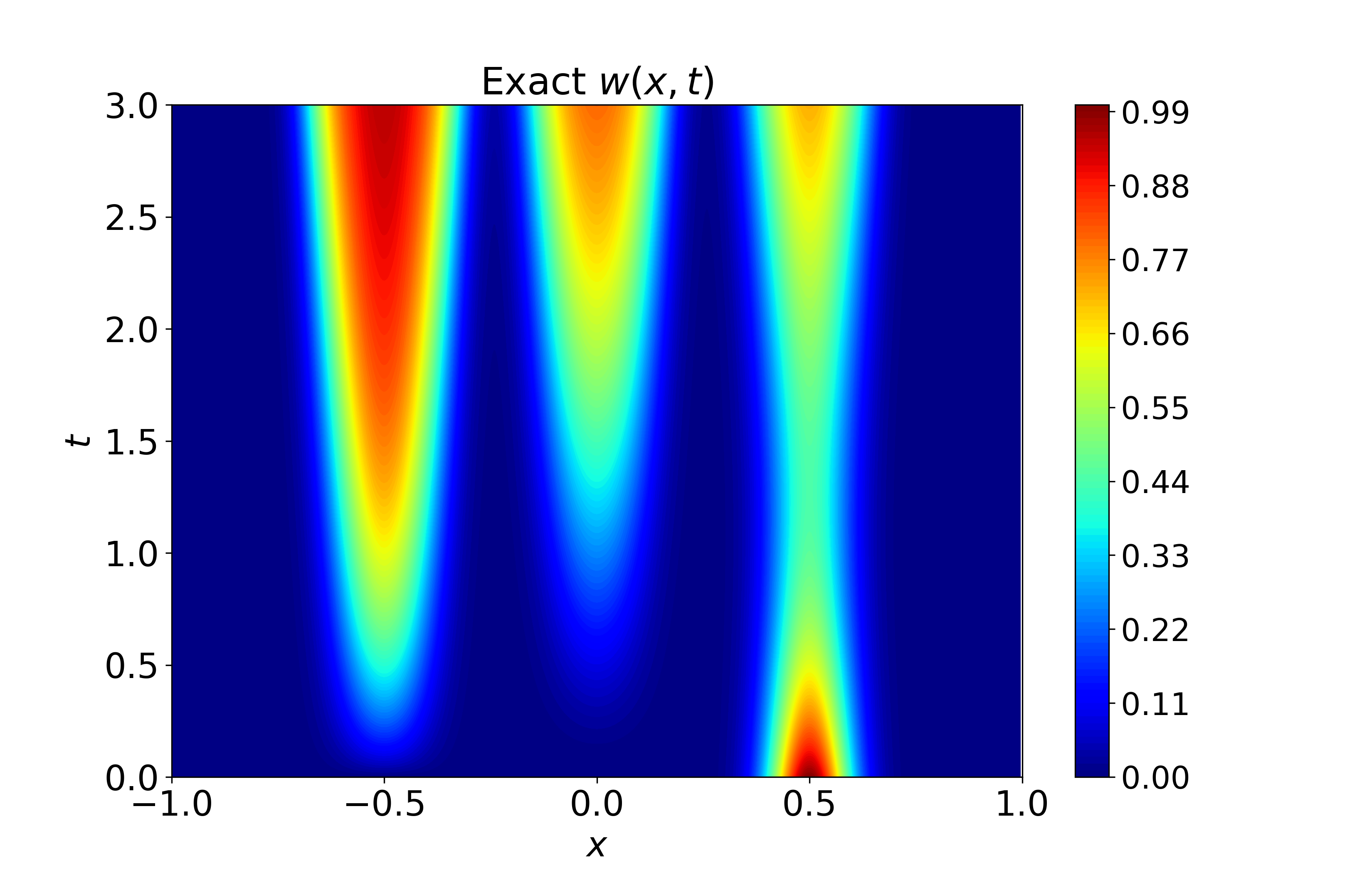} 
   \end{minipage}  
   \begin{minipage}{0.3\textwidth}
     \centering
     \includegraphics[width=\linewidth]{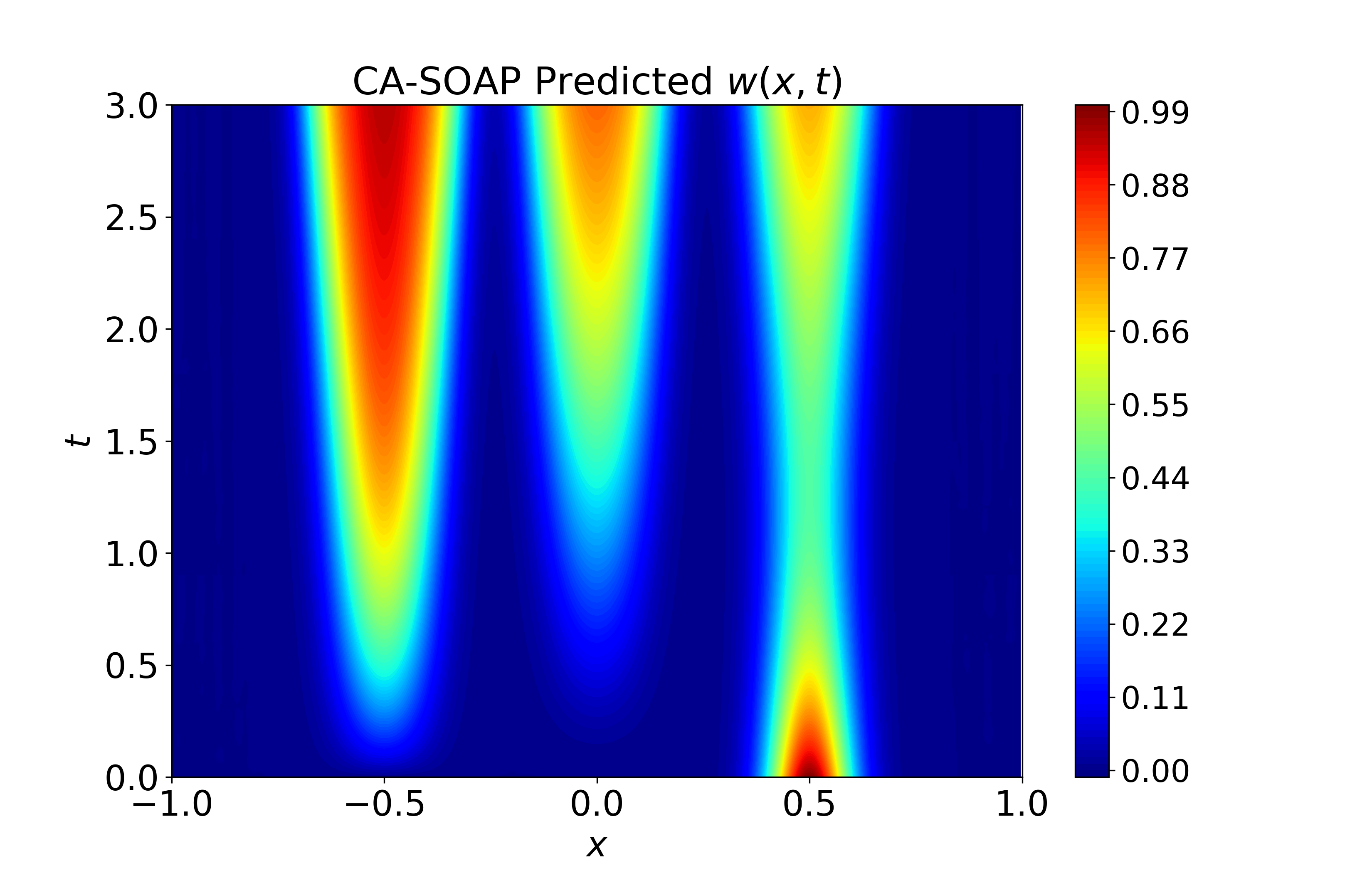} 
   \end{minipage} 
   \begin{minipage}{0.3\textwidth}
     \centering
     \includegraphics[width=\linewidth]{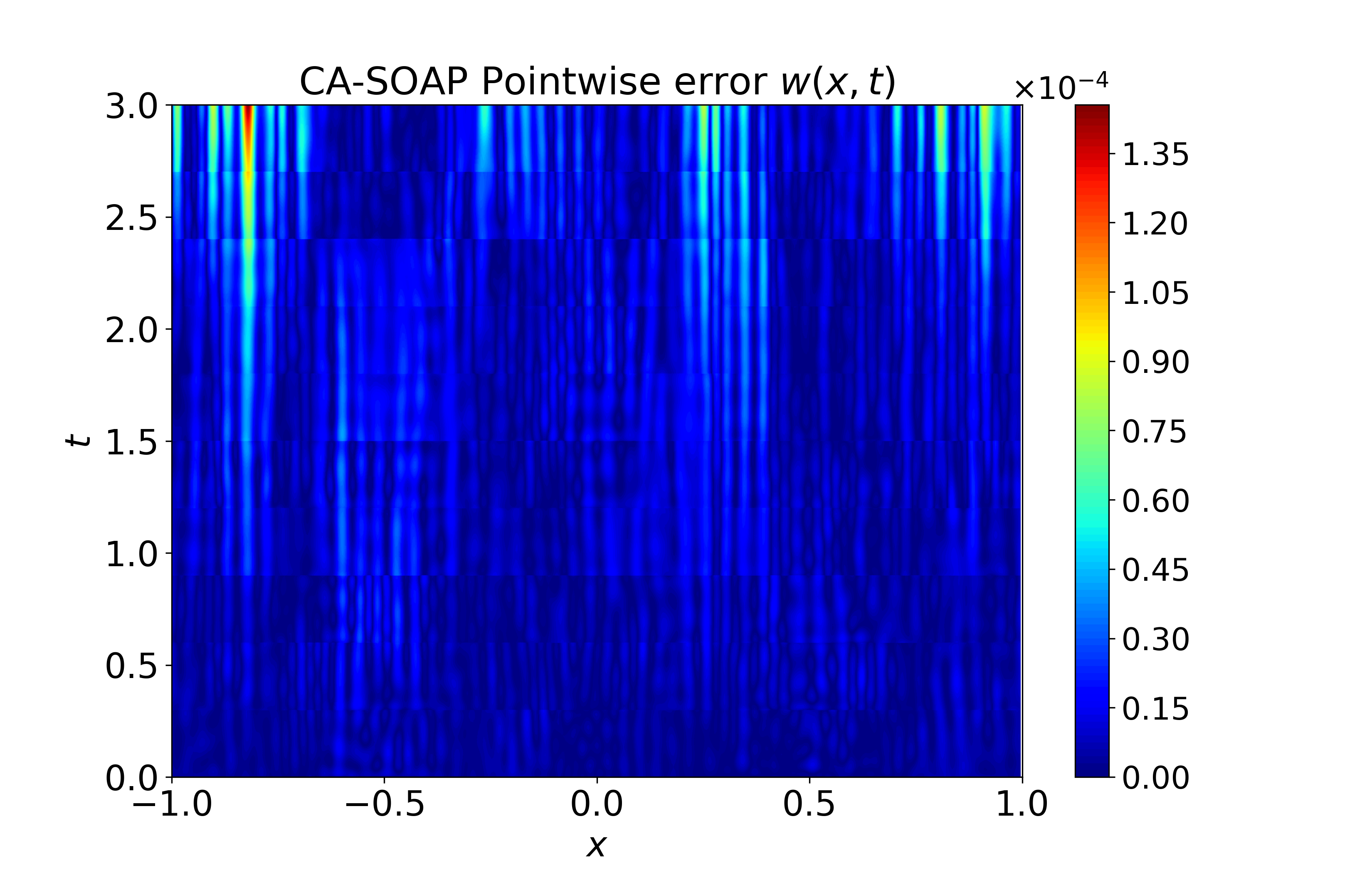} 
   \end{minipage} 
   
    \caption{Spatiotemporal heatmaps for the Belousov-Zhabotinsky system using CA-SOAP. For each species ($u$, $v$, and $w$), we show the reference solution, the prediction, and the corresponding point-wise error over the full spatiotemporal domain. The predicted fields accurately reproduce the coupled reaction-diffusion dynamics, with small errors throughout the domain.}\label{fig: BZ heatmap}
\end{figure}

\begin{figure}[!htb]
\centering
    \begin{minipage}{0.5\textwidth}
     \centering
     \includegraphics[width=\linewidth]{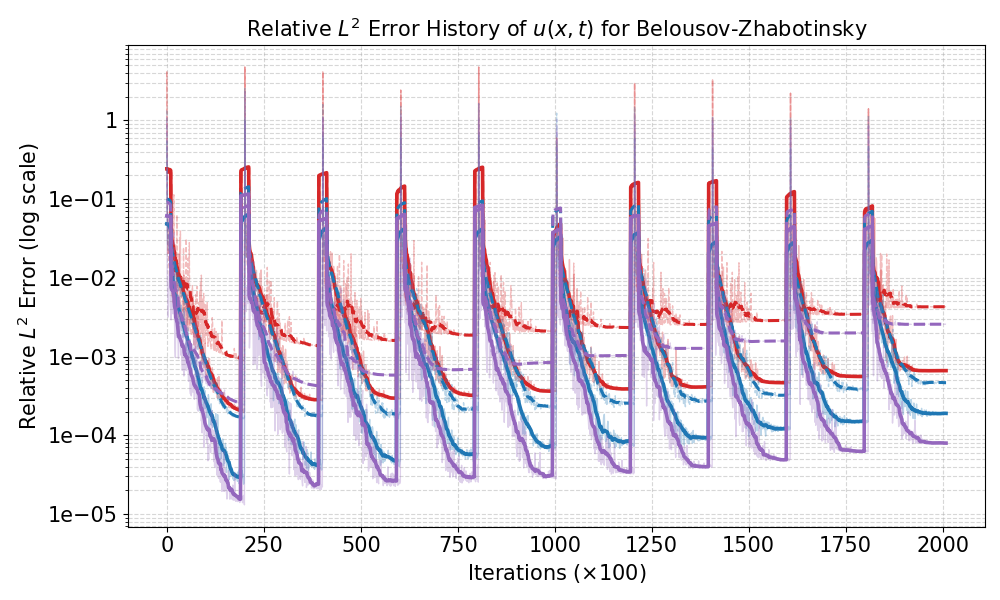} 
   \end{minipage}   
   
   \begin{minipage}{\textwidth}
     \centering
     \includegraphics[width=\linewidth]{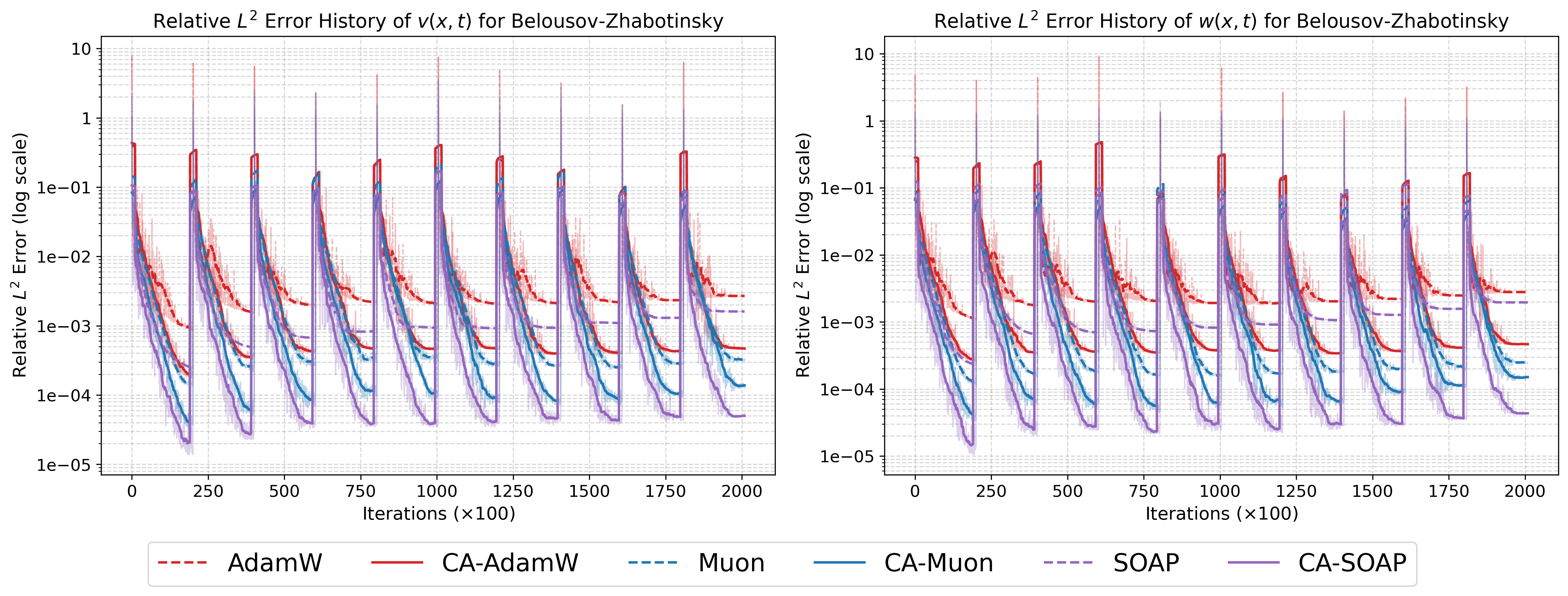} 
   \end{minipage}   
    \caption{History of the relative $L^2$ errors for the Belousov-Zhabotinsky system.}\label{fig: BZ Hist}
\end{figure}

\begin{table}[htbp]
    \centering
    \caption{\textbf{Errors for the Belousov-Zhabotinsky system (lower is better).} Average relative $L^2$ error and absolute $L^\infty$ error of $(u,~v,~w)$ for each optimizer, comparing the base method to its CA variant.}
    \label{tab: BZ table}
    
    \resizebox{\textwidth}{!}{
        \begin{tabular}{lcccccc}
        \toprule
        \textbf{Optimizers}
        & \multicolumn{3}{c}{Rel. $L^2$ error} 
        & \multicolumn{3}{c}{$L^\infty$ error} \\
        \cmidrule(lr){2-4} \cmidrule(lr){5-7}
        & $u$ & $v$ & $w$ & $u$ & $v$ & $w$ \\
        \hline

        \grow AdamW 
        & $3.08\times 10^{-3}$ & $2.11\times 10^{-3}$ & $2.30\times 10^{-3}$ 
        & $1.34\times 10^{-2}$ & $2.88\times 10^{-3}$ & $6.24\times 10^{-3}$ \\
        
        \brow CA-AdamW
        & $4.70\times 10^{-4}$ & $3.61\times 10^{-4}$ & $3.88\times 10^{-4}$ 
        & $1.73\times 10^{-3}$ & $3.89\times 10^{-4}$ & $8.19\times 10^{-4}$ \\
        
        \multicolumn{1}{l|}{\textbf{Error reduction (\%) $\uparrow$}}
        & 84.74\% & 82.89\% & 83.13\% & 87.09\% & 86.49\% & 86.88\% \\
    \hline
        
        \grow Muon
        & $3.31\times 10^{-4}$ & $2.58\times 10^{-4}$ & $2.51\times 10^{-4}$ 
        & $9.50\times 10^{-4}$ & $3.88\times 10^{-4}$ & $7.31\times 10^{-4}$  \\
        
        \brow CA-Muon
        & $8.96\times 10^{-5}$ & $6.41\times 10^{-5}$ & $6.70\times 10^{-5}$ 
        
        & $5.85\times 10^{-4}$ & $9.08\times 10^{-5}$ & \gbf{$\bf 1.22\times 10^{-4}$} \\
        
        \multicolumn{1}{l|}{\textbf{Error reduction (\%) $\uparrow$}}
            & 72.93\% & 75.16\% & 73.31\% & 38.42\% & 76.60\% & 83.31\% \\
        \hline
        
        \grow SOAP
        & $8.78\times 10^{-4}$ & $7.39\times 10^{-4}$ & $7.24\times 10^{-4}$ 
        & $5.33\times 10^{-3}$ & $9.69\times 10^{-4}$ & $1.97\times 10^{-3}$ \\
        
        \brow CA-SOAP
        & \gbf{$\bf 8.01\times 10^{-5}$} & \gbf{$\bf 4.66\times 10^{-5}$} & \gbf{$\bf 5.41\times 10^{-5}$} 
        
        & \gbf{$\bf 3.23\times 10^{-4}$} & \gbf{$\bf 6.19\times 10^{-5}$} & $1.50\times 10^{-4}$ \\
        
       \multicolumn{1}{l|}{\textbf{Error reduction (\%) $\uparrow$}}
            &\gbf{\bf 90.88\%} & \gbf{\bf 93.69\%} & \gbf{\bf 92.53\%} & \gbf{\bf 93.94\%} & \gbf{\bf 93.61\%} & \gbf{\bf 92.39\%} \\
        
        \bottomrule
        \end{tabular}
    }
    \vspace{-0.5em}
\end{table}

\subsection{2D Kuramoto-Sivashinsky System}
\label{Section: 2D KS system}
The Kuramoto-Sivashinsky (KS) equation combines nonlinear advection, fourth-order (bi-harmonic) dissipation, and an anti-diffusive term, and is well known for its chaotic spatiotemporal dynamics. It has been used extensively to model flame front and ion plasma instabilities. In this example, we consider the following 2D KS equation outlined in Ref.~\cite{rahman2025regularity}:
\begin{align*}
    \frac{\partial u}{\partial t} + \left( u \frac{\partial u}{\partial x} + v \frac{\partial u}{\partial y} \right) + \lambda \left( \frac{\partial^2 u}{\partial x^2} + \frac{\partial^2 u}{\partial y^2} \right) + \left( \frac{\partial^4 u}{\partial x^4} + 2\frac{\partial^4 u}{\partial x^2 \partial y^2} + \frac{\partial^4 u}{\partial y^4} \right) &= f_1(x,y,t), \\
    \frac{\partial v}{\partial t} + \left( u \frac{\partial v}{\partial x} + v \frac{\partial v}{\partial y} \right) + \lambda \left( \frac{\partial^2 v}{\partial x^2} + \frac{\partial^2 v}{\partial y^2} \right) + \left( \frac{\partial^4 v}{\partial x^4} + 2\frac{\partial^4 v}{\partial x^2 \partial y^2} + \frac{\partial^4 v}{\partial y^4} \right) &= f_2(x,y,t).
\end{align*}
An exact solution family parameterized by $\lambda = 0.01$ is:
\begin{align*}
    &u(x, y, t) = -\cos(\pi x)\sin(\pi y)\exp(-\frac{\pi^2\lambda t}{4}),\\
    &v(x, y, t) = \sin(\pi x)\cos(\pi y)\exp(-\frac{\pi^2\lambda t}{4}).
\end{align*}
The source functions $f_1(x,y,t)$ and $f_2(x,y,t)$ will be inferred by the exact solution above. Moreover, the spatial domain is $\Omega=[0,2]^2$ and time domain is $T = [0,1]$. The networks are constructed by $3$ hidden layers with $50$ neurons in each hidden layer. We take $N_f = 10{,}000$ and $N_b = N_0 = 200$ across all experiments, with $\lambda_F = 1$, and $\lambda_B = \lambda_I = 20$.

We run all optimizers for $20{,}000$ iterations, with the same learning rates and schedulers as in the previous experiments. The error history is shown in Fig.~\ref{fig: KS Hist}, and the numerical results are summarized in Table~\ref{tab: 2D KS table}.

Table~\ref{tab: 2D KS table} shows that the CA strategy remains consistently beneficial even for this challenging 2D high-order nonlinear system. Across all three optimizer families, the CA variants reduce both the relative $L^2$ and $L^\infty$ errors, indicating that the improvement is not tied to a specific optimizer design. In particular, CA-SOAP achieves the best overall performance, while CA-AdamW also yields substantial gains over its baseline. These results further support the generality of CA strategy as an effective enhancement for PINN training under complex spatiotemporal dynamics. Notably, the CA enhancement remains effective even when the baseline optimizer is already strong, as seen from the further improvement of SOAP on all reported metrics.

\begin{figure}[!htb]
\centering
   \begin{minipage}{\textwidth}
     \centering
     \includegraphics[width=\linewidth]{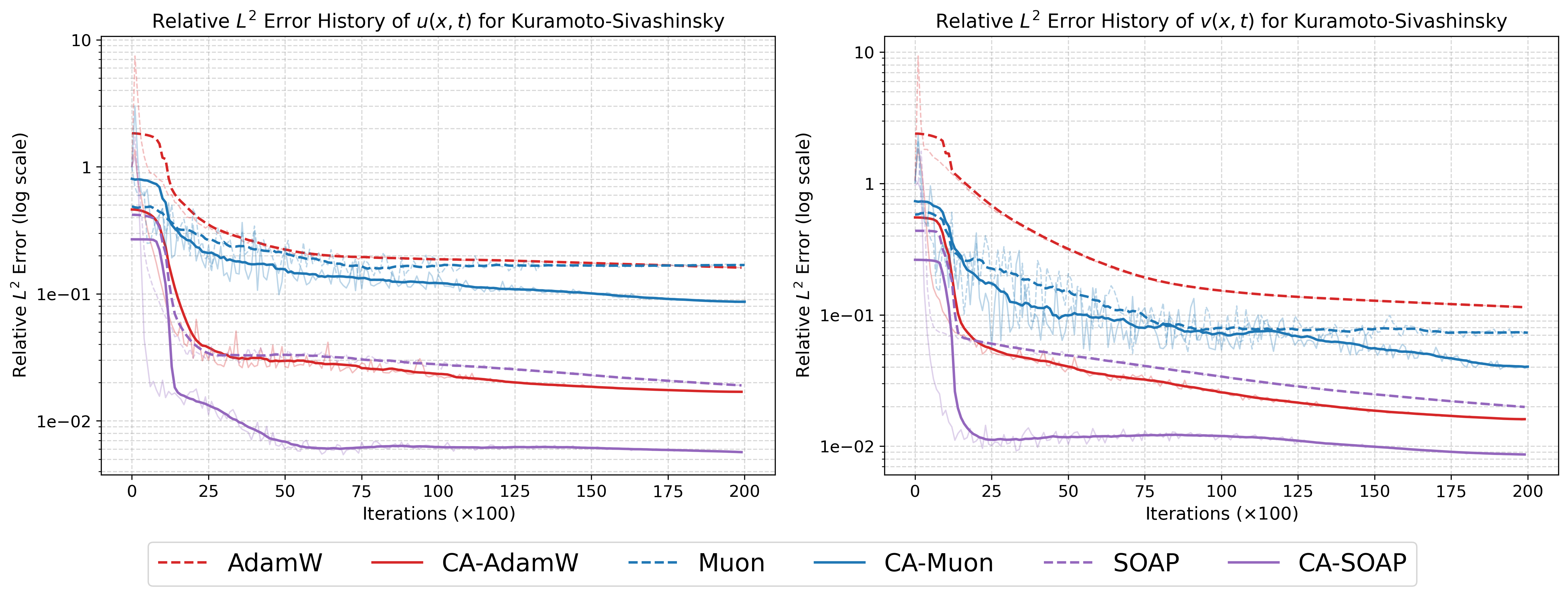} 
   \end{minipage}   
    \caption{History of the relative $L^2$ errors for the 2D Kuramoto--Sivashinsky system. The curves show the training evolution of the relative $L^2$ errors for both $u$ and $v$ under different optimizers.}\label{fig: KS Hist}
\end{figure}

\begin{table}[htbp]
    \centering
    \caption{\textbf{Errors for the 2D Kuramoto-Sivashinsky equation (lower is better).} Average relative $L^2$ error and absolute $L^\infty$ error of $(u,~v)$ for each optimizer, comparing the base method to its CA variant.}
    \label{tab: 2D KS table}
    
    \resizebox{0.85\textwidth}{!}{
        \begin{tabular}{lcccc}
        \toprule
        \textbf{Optimizers}
        & \multicolumn{2}{c}{Rel. $L^2$ error} 
        & \multicolumn{2}{c}{$L^\infty$ error} \\
        \cmidrule(lr){2-3} \cmidrule(lr){4-5}
        & $u$ & $v$ & $u$ & $v$ \\
        \hline

        \grow AdamW 
        & $1.44\times 10^{-1}$ & $9.61\times 10^{-2}$ & $1.37\times 10^{-1}$ & $2.19\times 10^{-1}$ \\
        
        \brow CA-AdamW
        & $1.79\times 10^{-2}$ & $1.76\times 10^{-2}$ & $5.66\times 10^{-2}$ & $4.16\times 10^{-2}$ \\
        
        \multicolumn{1}{l|}{\textbf{Error reduction (\%) $\uparrow$}}
            & \gbf{87.57\%} & \gbf{81.69\%} & 58.69\% & \gbf{81.00\%} \\
        \hline
        
        \grow Muon
        &$1.88\times 10^{-1}$ & $7.60\times 10^{-2}$ & $4.17\times 10^{-1}$ & $2.21\times 10^{-1}$ \\
        
        \brow CA-Muon
        & $8.70\times 10^{-2}$ & $3.99\times 10^{-2}$ & $1.20\times 10^{-1}$ & $9.29\times 10^{-2}$\\
        
        \multicolumn{1}{l|}{\textbf{Error reduction (\%) $\uparrow$}} 
            & 53.72\% & 47.50\% & 71.22\% & 57.96\% \\
        \hline
        
        \grow SOAP
         & $1.98\times 10^{-2}$ & $2.11\times 10^{-2}$ & $4.69\times 10^{-2}$ & $5.45\times 10^{-2}$ \\
        
        \brow CA-SOAP
        & \gbf{$\bf 4.70\times 10^{-3}$} & \gbf{$\bf 8.78\times 10^{-3}$} & \gbf{$\bf 9.37\times 10^{-3}$} & \gbf{$\bf 1.54\times 10^{-2}$} \\
        
        \multicolumn{1}{l|}{\textbf{Error reduction (\%) $\uparrow$}}
            & 76.26\% & 58.39\% & \gbf{80.02\%} & 71.74\% \\
        
        \bottomrule
        \end{tabular}
    }
    \vspace{-0.5em}
\end{table}

\subsection{Additional Comparison with Related Optimizers}
Beyond the baseline optimizers and their CA variants, we compare CA-AdamW with two related optimizer enhancement methods, Adan~\cite{xie2024adan} and ALTO~\cite{zhao2025exploring}. We focus on CA-AdamW because both methods are adaptive moment-based and thus are most fairly compared with it. Although Adan and ALTO also modify update dynamics beyond standard adaptive optimization, their mechanisms differ from ours: Adan uses gradient differences as an acceleration signal with fixed coefficients, and ALTO applies a fixed negative correction based on gradient differences. In contrast, our CA framework combines parameter displacement and gradient variation to construct a curvature-aware signal that dynamically adjusts the correction strength.

To further evaluate CA against related optimizer enhancement methods on a challenging PINN task, we conduct additional experiments on the 2D Kuramoto--Sivashinsky system, which involves high-order derivatives, nonlinear coupling, and complex spatiotemporal dynamics. As shown in Table~\ref{tab:ks_related_optimizers}, CA-AdamW consistently outperforms Adan and ALTO, indicating the benefit of its curvature-aware dynamic correction on difficult PINN optimization problems.
\begin{table}[htbp]
    \centering
   \caption{\textbf{Errors for the 2D Kuramoto--Sivashinsky equation (lower is better).} Average relative $L^2$ error and absolute $L^\infty$ error of $u$ and $v$ for CA-AdamW and two Adam-style optimizer enhancement methods, Adan and ALTO. \gbf{Green} denotes the best results.}

    \label{tab:ks_related_optimizers}
    
    \resizebox{0.85\textwidth}{!}{
        \begin{tabular}{lcccc}
        \toprule
        \textbf{Optimizers}
        & \multicolumn{2}{c}{Rel. $L^2$ error} 
        & \multicolumn{2}{c}{$L^\infty$ error} \\
        \cmidrule(lr){2-3} \cmidrule(lr){4-5}
        & $u$ & $v$ & $u$ & $v$ \\
        \hline

        Adan \cite{xie2024adan}
        & $4.77\times 10^{-2}$ & $4.73\times 10^{-2}$ & $8.99\times 10^{-2}$ & $9.32\times 10^{-2}$ \\

        ALTO \cite{zhao2025exploring}
        & $7.39\times 10^{-2}$ & $7.98\times 10^{-2}$ & $2.21\times 10^{-1}$ & $1.01\times 10^{-1}$ \\
        
        \brow CA-AdamW
        & \gbf{$\bf 1.79\times 10^{-2}$} & \gbf{$\bf1.76\times 10^{-2}$} & \gbf{$\bf5.66\times 10^{-2}$} & \gbf{$\bf 4.16\times 10^{-2}$} \\
        
        \bottomrule
        \end{tabular}
    }
    \vspace{-0.5em}
\end{table}
\section{Conclusion}
\label{sec:conclusions}
We proposed a curvature-aware optimization framework for PINNs that augments first-order optimizers with lightweight secant-based geometric adaptation. Motivated by the highly anisotropic and rapidly varying loss landscapes in PINN training, the framework uses consecutive gradient information to provide an adaptive predictive correction without explicitly forming second-order matrices.

Experiments on diverse PDE benchmarks show that the proposed method consistently improves convergence speed, training stability, and solution accuracy over standard optimizers, with particularly clear gains on challenging problems such as the 2D Kuramoto--Sivashinsky system.

Overall, these results indicate that lightweight secant-based curvature awareness is a practical and effective direction for PINN optimization. Future work includes strengthening the theoretical understanding of the adaptive mechanism and extending the framework to broader physics-informed learning settings, such as operator learning, inverse problems, and large-scale multiphysics systems.

\section*{Acknowledgments}
Ming Yan and Chenhao Si were partially supported by 
Guangdong Provincial Key Laboratory of Mathematical Foundations for Artificial Intelligence (2023B1212010001), Shenzhen Stability Science Program, and the Shenzhen Science and Technology Program under grant no. JCYJ20250604141043020.

\bibliographystyle{unsrt}  
\bibliography{ref}

\clearpage 
\appendix
\section{Algorithms for SOAP and Muon}
\label{append:Alogrithms for soap and muon}
\begin{figure}[!ht]
\centering
\begin{minipage}[t]{0.48\textwidth}
\begin{algorithm}[H]
\caption{SOAP Optimizer}
\label{alg:SOAP}

\textbf{Input:} Loss function $\ell$, initial parameters $\theta_0 \in \mathbb{R}^{m \times n}$, step sizes $\{\eta_k\}$, hyperparameters $\beta_1, \beta_2 \in [0,1)$ (momentum), $\beta_P \in [0,1)$ (preconditioner decay), $\varepsilon > 0$, $f$ (eigen-update frequency), $\lambda$ (weight decay)

\textbf{Output:} $\{\theta_k\}_{k=1}^T$.

\begin{enumerate}
    \item Initialize $\mathbf{m}_0 = \mathbf{0}$, $\mathbf{v}_0 = \mathbf{0}$, $\mathbf{L}_0 = \epsilon \mathbf{I}_m$, $\mathbf{R}_0 = \epsilon \mathbf{I}_n$, $\mathbf{U}_L = \mathbf{I}_m$, $\mathbf{U}_R = \mathbf{I}_n$
    
    \item \textbf{while} $k < T$ \textbf{do}
    
    \item \quad $\mathbf{G}_k = \nabla_\theta \ell (\theta_k, \zeta_k)$;
    
    \item \quad $\mathbf{L}_k = \beta_P \mathbf{L}_{k-1} + (1 - \beta_P) \mathbf{G}_k \mathbf{G}_k^\top$;

    \item \quad $\mathbf{R}_k = \beta_P \mathbf{R}_{k-1} + (1 - \beta_P) \mathbf{G}_k^\top \mathbf{G}_k$;
    \item \quad \textbf{if} $k \pmod f == 0$ \textbf{then}
    
    \item \quad \quad $\mathbf{U}_L, \mathbf{D}_L = \text{EigenDecomp}(\mathbf{L}_k)$; 
    
    \item \quad \quad $\mathbf{U}_R, \mathbf{D}_R = \text{EigenDecomp}(\mathbf{R}_k)$;
    
    \item \quad \textbf{end if}
    
    \item \quad $\tilde{\mathbf{G}}_k = \mathbf{U}_L^\top \mathbf{G}_k \mathbf{U}_R$;
    
    \item \quad $\mathbf{m}_k = \beta_1 \mathbf{m}_{k-1} + (1 - \beta_1) \tilde{\mathbf{G}}_k$; 
    
    \item \quad $\mathbf{v}_k = \beta_2 \mathbf{v}_{k-1} + (1 - \beta_2) \tilde{\mathbf{G}}_k^2$; 
    
    \item \quad $\hat{\mathbf{m}}_k = \mathbf{m}_k / (1 - \beta_1^k)$; \quad $\hat{\mathbf{v}}_k = \mathbf{v}_k / (1 - \beta_2^k)$;
    
    \item \quad $\mathbf{P}_k = \hat{\mathbf{m}}_k / (\sqrt{\hat{\mathbf{v}}_k} + \varepsilon)$;
    
    \item \quad $\Delta \theta_k = \mathbf{U}_L \mathbf{P}_k \mathbf{U}_R^\top$;
    
    \item \quad $\theta_{k+1} = \theta_k - \eta_k (\Delta \theta_k + \lambda \theta_k)$; 
    
    \item \textbf{end while}
\end{enumerate}
\end{algorithm}
\end{minipage}
\hfill
\begin{minipage}[t]{0.48\textwidth}
\begin{algorithm}[H]
\caption{Curvature-Aware SOAP}
\label{alg:CA-SOAP}

\textbf{Input:} Loss function $\ell$, initial parameters $\theta_0 \in \mathbb{R}^{m \times n}$, step sizes $\{\eta_k\}$, hyperparameters $\beta_1, \beta_2, \beta_a \in [0,1)$, $\beta_P \in [0,1)$, $\varepsilon > 0$, $f$ (eigen-update frequency), $\alpha_{\text{base}}$ (base trajectory), $\lambda$ (weight decay)

\textbf{Output:} $\{\theta_k\}_{k=1}^T$.

\begin{enumerate}
    \item Initialize $\mathbf{m}_0 = \mathbf{0}$, $\mathbf{v}_0 = \mathbf{0}$, $\mathbf{L}_0 = \epsilon \mathbf{I}_m$, $\mathbf{R}_0 = \epsilon \mathbf{I}_n$, $\mathbf{U}_L = \mathbf{I}_m$, $\mathbf{U}_R = \mathbf{I}_n$, \textcolor{red}{$\mathbf{A}_0 = \mathbf{0}$}, \textcolor{red}{$\mathbf{G}_0 = \mathbf{0}$, $\theta_{-1} = \theta_0$}
    
    \item \textbf{while} $k < T$ \textbf{do}
    
    \item \quad $\mathbf{G}_k = \nabla_\theta \ell (\theta_k, \zeta_k)$; 
    \item \quad \textcolor{red}{$\mathbf{A}_k = \beta_1 \mathbf{A}_{k-1} + (1 - \beta_a)(\mathbf{G}_k - \mathbf{G}_{k-1})$;}
    \item \quad \textcolor{red}{$\Delta \theta_k = \theta_k - \theta_{k-1}$;} 
    
    \item \quad \textcolor{red}{$\kappa_k = \frac{\langle \Delta \theta_k, \mathbf{G}_k - \mathbf{G}_{k-1} \rangle_F}{\|\Delta \theta_k\|_F^2}$;} 
    \item \quad \textcolor{red}{$\alpha_k = \alpha_{\text{base}} (1 + \tanh(-\kappa_k))$}
    
    \item \quad \textcolor{red}{$\mathbf{G}_{boost, k} = \mathbf{G}_k + \alpha_k \mathbf{A}_k$; }
    
    \item \quad $\mathbf{L}_k = \beta_P \mathbf{L}_{k-1} + (1 - \beta_P) \textcolor{red}{\mathbf{G}_{boost, k} \mathbf{G}_{boost, k}^\top}$; 
    
    \item \quad $\mathbf{R}_k = \beta_P \mathbf{R}_{k-1} + (1 - \beta_P) \textcolor{red}{\mathbf{G}_{boost, k}^\top \mathbf{G}_{boost, k}}$; 
    
    \item \quad \textbf{if} $k \pmod f == 0$ \textbf{then}
    
   \item \quad \quad $\mathbf{U}_L, \mathbf{D}_L = \text{EigenDecomp}(\mathbf{L}_k)$;
    
    \item \quad \quad $\mathbf{U}_R, \mathbf{D}_R = \text{EigenDecomp}(\mathbf{R}_k)$;
    
    \item \quad \textbf{end if}
    
    \item \quad $\tilde{\mathbf{G}}_k = \mathbf{U}_L^\top \textcolor{red}{\mathbf{G}_{boost, k}} \mathbf{U}_R$;
    
    \item \quad $\mathbf{m}_k = \beta_1 \mathbf{m}_{k-1} + (1 - \beta_1) \tilde{\mathbf{G}}_k$; 
    
    \item \quad $\mathbf{v}_k = \beta_2 \mathbf{v}_{k-1} + (1 - \beta_2) \tilde{\mathbf{G}}_k^2$; 
    
    \item \quad $\hat{\mathbf{m}}_k = \mathbf{m}_k / (1 - \beta_1^k)$; \quad $\hat{\mathbf{v}}_k = \mathbf{v}_k / (1 - \beta_2^k)$;
    
    \item \quad $\mathbf{P}_k = \hat{\mathbf{m}}_k / (\sqrt{\hat{\mathbf{v}}_k} + \varepsilon)$;
    
    \item \quad $\Delta_{step} = \mathbf{U}_L \mathbf{P}_k \mathbf{U}_R^\top$; 
    
    \item \quad $\theta_{k+1} = \theta_k - \eta_k (\Delta_{step} + \lambda \theta_k)$;
    
    \item \textbf{end while}
\end{enumerate}
\end{algorithm}
\end{minipage}
\end{figure}

\begin{figure}[!ht]
\centering
\begin{minipage}[t]{0.48\textwidth}
\begin{algorithm}[H]
\caption{Muon Optimizer}
\label{alg:Muon}

\textbf{Input:} Loss function $\ell$, initial parameters $\theta_0 \in \mathbb{R}^{m \times n}$ (assume $m \le n$), step sizes $\{\eta_k\}$, hyperparameter $\mu \in [0,1)$ (momentum), $\lambda$ (weight decay), $I$ (Newton-Schulz iterations)

\textbf{Output:} $\{\theta_k\}_{k=1}^T$.

\begin{enumerate}
    \item Initialize $\mathbf{M}_0 = \mathbf{0}$
    
    \item \textbf{while} $k < T$ \textbf{do}
    
    \item \quad $\mathbf{G}_k = \nabla_\theta \ell (\theta_k, \zeta_k)$; 
    
    \item \quad $\mathbf{M}_k = \mu \mathbf{M}_{k-1} + \mathbf{G}_k$;
    
    \item \quad $\mathbf{O}_k = \mu \mathbf{M}_k + \mathbf{G}_k$;
    
    \item \quad $\mathbf{X} = \mathbf{O}_k / \|\mathbf{O}_k\|_F$;
    
    \item \quad \textbf{for} $i = 1$ \textbf{to} $I$ \textbf{do} 
    
    \item \quad \quad $\mathbf{A} = \mathbf{X} \mathbf{X}^\top$;
    
    \item \quad \quad $\mathbf{X} = 1.5 \mathbf{X} - 0.5 \mathbf{A} \mathbf{X}$; 
    
    \item \quad \textbf{end for}
    
    \item \quad n$\Delta \theta_k = \mathbf{X} \cdot \max(1, \sqrt{n/m})$; 
    
    \item \quad $\theta_{k+1} = \theta_k - \eta_k (\Delta \theta_k + \lambda \theta_k)$; 
    
    \item \textbf{end while}
\end{enumerate}
\end{algorithm}
\end{minipage}
\hfill
\begin{minipage}[t]{0.48\textwidth}
\begin{algorithm}[H]
\caption{Curvature-Aware Muon}
\label{alg:CA-Muon}

\textbf{Input:} Loss function $\ell$, initial parameters $\theta_0 \in \mathbb{R}^{m \times n}$ (assume $m \le n$), step sizes $\{\eta_k\}$, hyperparameters $\mu \in [0,1)$ (momentum), $\alpha_{\text{base}}$ (base trajectory), $\lambda$ (weight decay), $I$ (Newton-Schulz iterations)

\textbf{Output:} $\{\theta_k\}_{k=1}^T$.

\begin{enumerate}
    \item Initialize $\mathbf{M}_0 = \mathbf{0}$, \textcolor{red}{$\mathbf{A}_0 = \mathbf{0}$}, \textcolor{red}{$\mathbf{G}_0 = \mathbf{0}$, $\theta_{-1} = \theta_0$}
    
    \item \textbf{while} $k < T$ \textbf{do}
    
    \item \quad $\mathbf{G}_k = \nabla_\theta \ell (\theta_k, \zeta_k)$; 
    
    \item \quad \textcolor{red}{$\mathbf{A}_k = \mu \mathbf{A}_{k-1} + (1 - \mu)(\mathbf{G}_k - \mathbf{G}_{k-1})$;}
    
    \item \quad \textcolor{red}{$\Delta \theta_k = \theta_k - \theta_{k-1}$;} 
    
    \item \quad \textcolor{red}{$\kappa_k = \frac{\langle \Delta \theta_k, \mathbf{G}_k - \mathbf{G}_{k-1} \rangle_F}{\|\Delta \theta_k\|_F^2}$;} 
    
    \item \quad \textcolor{red}{$\alpha_k = \alpha_{\text{base}} (1 + \tanh(-\kappa_k))$} 
    
    \item \quad $\textcolor{red}{\mathbf{G}_{boost, k} = \mathbf{G}_k + \alpha_k \mathbf{A}_k}$; 
    
    \item \quad $\mathbf{M}_k = \mu \mathbf{M}_{k-1} + \textcolor{red}{\mathbf{G}_{boost, k}}$; 
    
    \item \quad $\mathbf{O}_k = \mu \mathbf{M}_k + \textcolor{red}{\mathbf{G}_{boost, k}}$; 
    
    \item \quad $\mathbf{X} = \mathbf{O}_k / \|\mathbf{O}_k\|_F$;
    
    \item \quad \textbf{for} $i = 1$ \textbf{to} $I$ \textbf{do} 
    
    \item \quad \quad $\mathbf{A} = \mathbf{X} \mathbf{X}^\top$; 
    
    \item \quad \quad $\mathbf{X} = 1.5 \mathbf{X} - 0.5 \mathbf{A} \mathbf{X}$; 
    
    \item \quad \textbf{end for}
    
    \item \quad $\Delta_{step} = \mathbf{X} \cdot \max(1, \sqrt{n/m})$;
    
    \item \quad $\theta_{k+1} = \theta_k - \eta_k (\Delta_{step} + \lambda \theta_k)$; 
    
    \item \textbf{end while}
\end{enumerate}
\end{algorithm}
\end{minipage}
\end{figure}
\end{document}